\documentclass[preprint,numbers]{elsarticle}
\usepackage{cite}

\usepackage{amsmath,amssymb,amsfonts}
\usepackage{graphicx}
\usepackage{textcomp}
\usepackage{xcolor}
\usepackage{outlines}
\usepackage{stfloats}
\usepackage[inline]{enumitem}

\usepackage{mathtools,xparse}
\DeclarePairedDelimiter{\abs}{\lvert}{\rvert}

\usepackage{xcolor}
\usepackage[latin9]{inputenc}
\usepackage{xcolor}
\usepackage{array}
\usepackage{verbatim}
\usepackage{algorithm}%,algorithmic}
\usepackage{algpseudocode}
\algblock{ParFor}{EndParFor}
\algnewcommand\algorithmicparfor{\textbf{for}}
\algnewcommand\algorithmicpardo{\textbf{pardo}}
\algnewcommand\algorithmicendparfor{\textbf{end for}}
\algrenewtext{ParFor}[1]{\algorithmicparfor\ #1\ \algorithmicpardo}
\algrenewtext{EndParFor}{\algorithmicendparfor}
\makeatletter
\ifthenelse{\equal{\ALG@noend}{t}}%
  {\algtext*{EndParFor}}
  {}%
\makeatother

\algnewcommand{\LineComment}[1]{\State \(\textcolor{blue}{\triangleright}\) #1}
\usepackage{amsmath}
\usepackage{amsthm}
\usepackage{booktabs}
\usepackage{amssymb}
\usepackage{graphicx}
\usepackage{outlines}
\usepackage{mathtools}
\usepackage{caption}
\usepackage{subcaption}
\usepackage{braket}
\usepackage{xcolor}
\usepackage{multirow}

\usepackage{bm}
\PassOptionsToPackage{normalem}{ulem}
\usepackage{ulem}
\makeatletter
\def\Let@{\def\\{\notag\math@cr}}
\makeatother
\usepackage{mathtools}

\usepackage{tikz}
\usetikzlibrary{intersections,calc}
\usetikzlibrary{matrix}
\usetikzlibrary{patterns}
\tikzset{axis break gap/.initial=1mm}
\usepackage{pgfplots}
\usepgfplotslibrary{groupplots}
\usepgfplotslibrary{external}

\makeatletter
\newcommand\resetstackedplots{
\makeatletter
\pgfplots@stacked@isfirstplottrue
\makeatother
\addplot [forget plot,draw=none] coordinates{(1,0) (5,0) (10,0) (20,0) (40,0)};
}

\newtheorem{thm}{Theorem}

\newdefinition{rmk}{Remark}
\newproof{pf}{Proof}
\newproof{pot}{Proof of Theorem \ref{thm2}}

\newcommand{\trans}{\mathsf{T}}
\newcommand{\size}[1]{\left\lvert #1 \right\rvert}
\newcommand{\paren}[1]{\left( #1 \right)}
\newcommand{\func}[2][]{
    \ifthenelse{\equal{#2}{}}
    {\text{\fontfamily{lmtt}\selectfont #1}}
    {\text{\fontfamily{lmtt}\selectfont #1}\left(#2\right)}}
\newcommand{\sample}[1][]{
    \ifthenelse{\equal{#1}{}}
    {\text{\fontfamily{lmtt}\selectfont SAMPLE}}
    {\text{\fontfamily{lmtt}\selectfont SAMPLE}\left(#1\right)}}
\newcommand{\G}[1][]{
    \ifthenelse{\equal{#1}{}}
    {\mathcal{G}}
    {\mathcal{G}_{#1}}}
\newcommand{\V}[1][]{
    \ifthenelse{\equal{#1}{}}
    {\mathcal{V}}
    {\mathcal{V}_{#1}}}
\newcommand{\E}[1][]{
    \ifthenelse{\equal{#1}{}}
    {\mathcal{E}}
    {\mathcal{E}_{#1}}}
\newcommand{\Gt}[1][]{
    \ifthenelse{\equal{#1}{}}
    {\mathcal{G}}
    {\mathcal{G}}_{\text{#1}}}
\newcommand{\Vt}[1][]{
    \ifthenelse{\equal{#1}{}}
    {\mathcal{V}}
    {\mathcal{V}_{\text{#1}}}}
\newcommand{\Et}[1][]{
    \ifthenelse{\equal{#1}{}}
    {\mathcal{E}}
    {\mathcal{E}_{\text{#1}}}}

\newcommand{\f}[1][]{
    \ifthenelse{\equal{#1}{}}
    {f}
    {f^{\paren{#1}}}}
\newcommand{\X}[1][]{
    \ifthenelse{\equal{#1}{}}
    {\bm{X}}
    {\bm{X}^{\paren{#1}}}}
\newcommand{\Ws}[1][]{
    \ifthenelse{\equal{#1}{}}
    {\bm{W}_\circ}
    {\bm{W}_\circ^{\paren{#1}}}}
\newcommand{\Wn}[1][]{
    \ifthenelse{\equal{#1}{}}
    {\bm{W}_\star}
    {\bm{W}_\star^{\paren{#1}}}}
\newcommand{\W}[1][]{
    \ifthenelse{\equal{#1}{}}
    {\bm{W}}
    {\bm{W}_{\text{#1}}}}
\newcommand{\An}[1][]{
    \ifthenelse{\equal{#1}{}}
    {\Tilde{\bm{A}}}
    {\Tilde{\bm{A}}_{\text{#1}}}}

\begin{document}

\title{Accurate, Efficient and Scalable Training of Graph Neural Networks}
\author[1]{Hanqing Zeng\corref{cor1}}
\ead{zengh@usc.edu}
\author[1]{Hongkuan Zhou\corref{cor1}}
\ead{hongkuaz@usc.edu}
\author[1]{Ajitesh Srivastava}
\ead{ajiteshs@usc.edu}
\author[2]{Rajgopal Kannan}
\ead{rajgopal.kannan.civ@mail.mil}
\author[1]{Viktor Prasanna}
\ead{prasanna@usc.edu}

\cortext[cor1]{Equal contribution.}
\address[1]{University of Southern California, Los Angeles, CA}
\address[2]{US Army Research Lab, Los Angeles, CA}

\newpageafter{author}

\begin{abstract}
{Graph Neural Networks (GNNs) are powerful deep learning models to generate node embeddings on graphs. 
When applying deep GNNs on large graphs, it is still challenging to perform training in an efficient and scalable way.
We propose a novel parallel training framework. Through sampling small subgraphs as minibatches, we reduce training workload by orders of magnitude compared with state-of-the-art minibatch methods. 
We then parallelize the key computation steps on tightly-coupled shared memory systems.
\color{black}}
For graph sampling, we exploit parallelism within and across sampler instances, and propose an efficient data structure supporting concurrent accesses from samplers. 
The parallel sampling algorithm theoretically achieves near-linear speedup with respect to number of processing units. 
For feature propagation within subgraphs, we improve cache utilization and reduce DRAM traffic by data partitioning. 
Our partitioning is a 2-approximation strategy for minimizing the communication cost compared to the optimal. 
We further develop a runtime scheduler to reorder the training operations and adjust the minibatch subgraphs for better parallel performance. 
{Finally, we generalize the above parallelization strategies to support multiple types of GNN models and graph samplers. 
The proposed graph embedding method outperforms the state-of-the-art in scalability, efficiency and accuracy simultaneously. }
On a 40-core Xeon platform, we achieve $60\times$ speedup (with AVX enabled) in the sampling step and $20\times$ speedup in the feature propagation step, compared to the serial implementation. Our algorithm enables fast training of deeper GNNs, as demonstrated by orders of magnitude speedup compared to state-of-the-art Tensorflow implementation.
\end{abstract}
\begin{keyword}
Graph representation learning; Graph Neural Networks; Graph sampling; Graph partitioning; Memory optimization;
\end{keyword}

\maketitle
\newpage
% NSF 3D memory
\section{Introduction}

% subgraph sampling vs. level sampling (GCN/GraphSAGE/FastGCN)
% sampling: accuracy/work-efficient in large graphs
% similarity and difference between CNN and GCN

Graph embedding is a powerful dimensionality reduction technique to facilitate downstream graph analytics. 
The embedding process converts graph nodes with unstructured neighbor connections into points in a low-dimensional vector space.
%Taking an unstructured, attributed graph as input, the embedding process outputs structured vectors which capture information of the original graph. 
%The goal of graph embedding is to perform dimensionality reduction by representing unstructured, attributed graphs as structured vectors. 
Embedding is essential for a wide range of tasks such as content recommendation \citep{pinsage}, traffic forecasting \citep{gcn_traffic}, image recognition \citep{gcn_cv} and protein function prediction \citep{graphsage}. Among the various embedding techniques, {Graph Neural Networks (GNNs)} (including Graph Convolutional Network (GCN) \citep{gcn} 
and its variants \citep{graphsage}, \citep{fastgcn}, \citep{gat}, \citep{graphsaint}) have attained much attention. {GNNs} produce accurate and robust embedding without the need of manual feature selection.

%Challenges exist to parallelize GCN models on large scale graphs. 

On large graphs, {GNN} training proceeds in the unit of minibatches. 
Due to edge connections, the graph nodes are not I.I.D distributed, and thus cannot be sampled uniformly at random as minibatch data points. 
State-of-the-art methods construct minibatches by sampling on each {GNN} layer (i.e., \emph{layer sampling}).
%To scale GCN training to large graphs, the typical approach is to decompose training into ``minibatches" and attempt to parallelize minibatch computation by sampling on GCN layers (layer sampling). 
The {vanilla GCN~\citep{gcn}} and its successor GraphSAGE~\citep{graphsage} sample by tracking down the inter-layer connections. Their approaches preserve the training accuracy of the original model, but the parallel training is not work-efficient {due to a phenomenon often referred to as ``neighbor explosion'' \citep{graphsage, sgcn, fastgcn}. Namely, 
for every additional GNN layer traversed by their samplers, the number of sampled nodes (i.e., neighbors) grows by an order of magnitude. 
Consequently, the sampled nodes across different minibatches overlap significantly, especially at the first few GNN layers. The amount of redundant computation thus increases exponentially with the number of GNN layers. }
To alleviate such high redundancy, FastGCN~\citep{fastgcn} proposes to independently sample the nodes of each GNN layer, without explicitly considering the layer connection constraint. Although FastGCN is faster than \citep{gcn, graphsage}, it incurs significant accuracy loss and requires preprocessing on the full grpah which is expensive and not easily parallelizable.

Due to the layer sampling design philosophy, it is difficult for state-of-the-art methods \citep{gcn,graphsage,fastgcn} to simultaneously achieve accuracy, efficiency and scalability. In this work, we perform sampling on the graph rather than the {GNN} layers. 
%propose a new \emph{graph sampling}-based minibatch method which achieves superior performance on a variety of large graphs. 
Our novelty lies in proposing a \emph{graph sampling}-based minibatch training algorithm via joint optimization on the learning quality and parallelization cost.
%scaling GCN training based on parallelized graph sampling (rather than layer sampling), without compromising accuracy or efficiency. 
We achieve scalability by 1) Developing a novel data structure that enables efficient subgraph sampling through supporting fast parallel updates on the sampling probability; 2) Optimizing parallel execution of intra-subgraph feature propagation and layer-wise weight updates --- specifically a cache-efficient subgraph partitioning scheme that guarantees near-minimal DRAM traffic. 
{Optimization in the above two steps can be generalized to support multiple GNN models and sampling algorithms.}
%3) Devising a template for exploiting {\it inter-subgraph} parallelism for various graph sampling techniques at very low cost. 
We achieve work-efficiency by avoiding ``neighbor explosion'', as each layer of our minibatched {GNN} contains the same number of neurons corresponding to the subgraph nodes. 
%For each iteration of weight updates, we sample a small subgraph out of the large training graph and construct a complete GCN from the subgraph. Since the neural network in our case is small yet complete, we can avoid neighbor explosion. 
Finally, we achieve learning accuracy since our sampled subgraphs preserve connectivity characteristics of the original training graph.
%To scale GCNs to large graphs, sampling is a necessary step. However, neighbor explosion is inevitable as long as the sampling belongs to the category of level-sampling. 
%Problem with state-of-the-art GCN based methods is that they develop training algorithms from the view point of a CNN on structured input data, without leveraging the properties of the underlying graph structure. 
The main contributions of this paper are:

\begin{outline}
\1 We propose a parallel GNN training algorithm based on graph sampling:
    \2 \emph{Accuracy} is achieved since the sampler returns small, representative subgraphs of the original graph. %due to connectivity-preserving graph sampling techniques.
    \2 \emph{Efficiency} is {optimized} since we always build complete GNNs on the minibatch subgraphs to avoid ``neighbor explosion" in deeper layers.
    \2 \emph{Scalability} is achieved with respect to number of processing cores, graph size and {GNN} depth by parallelizing various key steps.
\1 We propose a novel data structure that supports fast, incremental and parallel updates to a probability distribution. Our parallel sampler based on this data structure theoretically and empirically achieves near-linear scalability with respect to number of processing units. % evaluate their impact on embedding accuracy and speed. 
\1 We parallelize all the key operations to scale the overall minibatch training to a large number of processing cores. Specifically, for subgraph feature propagation, we perform intelligent partitioning along the feature dimension to achieve close-to-optimal DRAM and cache performance.
\1 We propose a runtime scheduling algorithm for training: 
    \2 By rearranging the order of various operations, we significantly reduce the training time under a wide range of model configurations. 
    \2 By partition scheduling and node clipping of subgraphs, we improve the feature propagation performance by better cacheline alignment.  
    %\2 Before dispatching the subgraphs to the minibatched GCNs, we perform node clipping to ensure cacheline alignment (and thus better performance) during feature propagation. 
    %\2 Our scheduler is a general one, which can be applied to training under any other graph sampling algorithm. 
\1 {We show that our parallelization and scheduling techniques are applicable to a number of GNN architectures (including graph convolution and graph attention) and graph sampling algorithms (including random edge sampling and variants of random walk sampling). }
\1 We perform thorough evaluation on a 40-core Xeon server. Compared with serial implementation, we achieve $15\times$ overall training time speedup. Compared with state-of-the-art minibatch methods, our training achieves up to $7.8\times$ speedup without accuracy loss.
\1 Our parallel training greatly facilitates development of deeper GNN models on larger graphs. We achieve two orders of magnitude speedup for 3-layer GNNs compared to state-of-the-art Tensorflow implementation.
%\1 We propose a subgraph-sampling based graph convolutional network model, which performs highly accurate graph embedding in a \emph{scalable} and \emph{work efficient} manner.
%	\2 We propose several subgraph-sampling techniques, and evaluate their impact on the quality and speed of graph embedding.
%    \2 Our sampled-subgraph based method leads to accuracy improvement and significant training time reduction compared with state-of-the-art embedding algorithms. 
%\1 For the proposed subgraph-sampling techniques, we parallelize them in a scalable and cache-efficient way. 
%\1 We optimize both the training and inference computation, by enhancing cache efficiency and load balancing using graph partitioning. 
\end{outline}
\section{Background and Related Work}

%\subsection{Graph Embedding Techniques}

%[incomplete]

%\begin{itemize}
%    \item Summarize types of embedding: deeping learning (GCN or DeepWalk); graph kernel; matrix factorization
%    \item Advantage of GCN based solutions.
%\end{itemize}

%[incomplete](Very high level]) Some intuition of graph convolution. Make the connection with CNNs.

\label{sec: gcn background}

{Graph Neural Networks (GNNs), including Graph Convolutional Network (GCN) \citep{gcn}, GraphSAGE \citep{graphsage} and Graph Attention Network (GAT) \citep{gat}, are the state-of-the-art deep learning models for graph embedding. They have been widely shown to learn highly accurate and robust representations of the graph nodes. }
{Like CNNs, GNNs} belong to a type of multi-layer neural network, which performs node embedding as follows. 
{The input to a GNN is a graph whose each node is associated with a feature vector (i.e., node attribute).} 
The {GNN} propagates the features of each node layer by layer, where each layer performs tensor operations based on the  model weights and the input graph topology. The last {GNN} layer outputs embedding vectors for each node of the input graph. 
Essentially, both the input node attributes and the topological information of the graph are ``embedded" into the output vectors. 

%%%%%%%
% The embedding is generated by graph convolutional layers. After generating the embedding, we can apply appropriate machine learning models on the embedding vectors for the target application. 
% For example, for node classification, ...
%%%%%%%%%%

%the embedding vectors capture the information of both the input graph attributes and topology. 
%Thus, the GCN needs to be constructed such that the feature propagation along layers reflects the information diffusion pattern within the input graph. In the following, we formally define the GCN structure and operations. 

\begin{figure*}[ht]
\begin{center}
\includegraphics[width=0.8\textwidth]{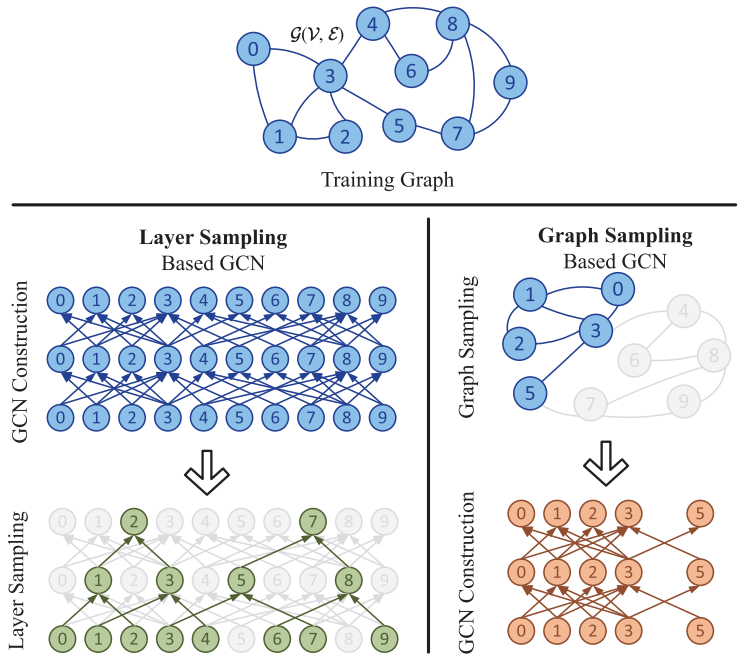}
\caption{Illustration on layer sampling and graph sampling based GCN design.}
\label{fig: illustration}
\end{center}
\end{figure*}

\subsection{Forward and Backward Propagation}

%In the following, we formally define the forward and backward propagation operations of GCN layers. 
{In this paper, we mainly consider four types of widely used GNNs: Graph Convolutional Network (GCN) \citep{gcn}, GraphSAGE \citep{graphsage}, MixHop \citep{mixhop} and Graph Attention Network (GAT) \citep{gat}. We first introduce in detail the GraphSAGE model architecture, and then summarize the layer operations of the other three. }

Let the input graph be $\G\paren{\V,\E,\X}$, where $\X\in \mathbb{R}^{\size{\V}\times f}$ stores the initial node attributes, and $f$ is the initial feature length. 
%Figure \ref{fig: illustration} shows the network structure of GCNs. 
{A GraphSAGE layer aggregates signals of nodes $\V$ along the edges $\E$. 
A full GraphSAGE network is build by stacking multiple layers, where the inputs to the next layer are the outputs of the previous one.} We use superscript ``$(\ell)$" to denote GNN layer-$\ell$ parameters. 
For a layer $\ell$, it contains  $\size{\V}$ nodes corresponding to the graph nodes. 
Each input and output node of the layer is associated with a feature vector of length $\f[\ell-1]$ and $\f[\ell]$, respectively. 
Denote $\X[\ell-1]\in\mathbb{R}^{\size{\V}\times \f[\ell-1]}$ and $\X[\ell]\in\mathbb{R}^{\size{\V}\times \f[\ell]}$ as the input and output feature matrices of the layer, where $\X[0]=\X$ and $\f[0]=\f$. 
A layer input node $v^{\paren{\ell-1}}$ is connected to a layer output node $u^{\paren{\ell}}$ if and only if $\paren{v,u}\in\E$. 
If we view the input and output nodes as a bipartite graph, then the bi-adjacency matrix $\bm{A}^{\paren{\ell}}$ equals the adjacency matrix $\bm{A}$ of $\G$.

Each GraphSAGE layer contains two learnable weight matrices: self-weight $\Ws$ the neighbor-weight $\Wn$. The forward propagation of a layer is defined by:

\begin{equation}
\label{eq: graphsage forward}
    \X[\ell] = \func[ReLU]{\An\cdot \X[\ell-1]\cdot \Wn[\ell] \Big\| \X[\ell-1]\cdot \Ws[\ell]}
\end{equation}

\noindent where ``$\|$'' is the column-wise matrix concatenation operation, and $\An$ is the normalized adjacency matrix. {The normalization can be calculated as $\An=\bm{D}^{-1}\cdot \bm{A}$, where $\bm{A}$ is the binary adjacency matrix of $\G$ and $\bm{D}$ is the diagonal degree matrix of $\bm{A}$ (i.e., $D_{ii}=\func[deg]{i}$).}

From Equation \ref{eq: graphsage forward}, each layer performs two key operations:

\begin{enumerate}
    \item \emph{Feature aggregation}: Each layer-$\ell$ node collects features of its layer-$\paren{\ell-1}$ neighbors and then calculates the weighted sum, as shown by $\An\cdot \X[\ell-1]$. 
    \item \emph{Weight transformation}: The aggregated neighbor features are multiplied by $\Wn[\ell]$. The features of a layer-$\paren{\ell-1}$ node itself are multiplied by $\Ws[\ell]$. 
\end{enumerate}

After obtaining the node embedding from the outputs of the last GNN layer, we can further perform various downstream tasks by analyzing the embedding vectors. For example, we can use a simple Multi-Layer Perceptron (MLP) to classify the graph nodes into $C$ classes. 
Let $L$ be the total number of GNN layers. So $\X[L]$ is the final node embedding. 
Following the design of \citep{graphsage,fastgcn,asgcn}, the classifier MLP generates the node prediction by: 

\begin{align}
\label{eq: mlp forward}
    \X_\text{MLP} =& \func[ReLU]{\X[L]\cdot \W[MLP]}\\
    \bm{Y} =& \sigma\paren{\X_\text{MLP}}
\end{align}

\noindent where $\W[MLP]\in \mathbb{R}^{\f[L]\times C}$. Function $\sigma\paren{\cdot}$ is the element-wise sigmoid or row-wise softmax to generate the probability of a node belonging to a class. 

Under the supervised learning setting, each node of $\V$ is also provided with the ground-truth class label(s).
Let $\overline{\bm{Y}}\in\mathbb{R}^{\size{\V}\times C}$ be the binary matrix encoding the ground-truth labels. 
Comparing the prediction with the ground-truth, we can obtain a scalar loss value, $\mathcal{L}$, by cross-entropy (CE):

\begin{equation}
\label{eq: loss forward}
    \mathcal{L}=\func[CE]{\bm{Y},\overline{\bm{Y}}}
\end{equation}

{For the other three types of GNNs under consideration, we need to update Equation \ref{eq: graphsage forward} for different forward propagation rules. Specifically, for GCN \citep{gcn}, the main difference from GraphSAGE is that there is not an explicit term $\X[\ell-1]\cdot \Ws[\ell]$ to capture the influence of a node to itself. Instead, the self-influence is propagated by adding a self-connection in the graph. Therefore, the adjacency matrix becomes $\bm{I}+\bm{A}$ and the normalization is performed differently. The forward propagation of each layer is as follows: 

\begin{equation}
\label{eq: gcn forward}
\X[\ell] = \func[ReLU]{\hat{\bm{A}}\cdot \X[\ell-1]\cdot \bm{W}^{(\ell)}}
\end{equation}
where $\hat{\bm{A}}$ is a symmetrically normalized adjacency matrix calculated by $\hat{\bm{A}}=\paren{\bm{I}+\bm{D}}^{-\frac{1}{2}}\cdot \paren{\bm{I}+\bm{A}}\cdot\paren{\bm{I}+\bm{D}}^{-\frac{1}{2}}$, and $\bm{I}$ is the identity matrix. 

For MixHop \citep{mixhop}, each layer is able to propagate influence from nodes up to $K$-hops away (i.e., $u$ is said to be $K$-hops away from $v$ if the shortest path from $u$ to $v$ has length $K$). The forward propagation of each layer is defined as:

\begin{equation}
\label{eq: mixhop forward}
\X[\ell] = \func[ReLU]{\Big\|_{k=0}^K \hat{\bm{A}^k}\cdot \X[\ell-1]\cdot \bm{W}_k^{(\ell-1)}}
\end{equation}
where ``$\|$'' is again the operation for matrix concatenation. $\hat{\bm{A}}^k$ means the symmetrically normalized adjacency matrix raised to the power of $k$. And ``order'' $K$ is a hyperparameter of the model. 

For GAT \citep{gat}, instead of aggregating the features from the previous layer (i.e., $\X[\ell-1]$) using a fixed adjacency matrix (i.e., $\hat{\bm{A}}$ in GCN or $\An$ in GraphSAGE), each GAT layer learns the weight of the adjacency matrix as the ``attention''. The forward propagation of a GAT layer is specified as:

\begin{equation}
\label{eq: gat forward}
\X[\ell] = \func[ReLU]{\bm{A}_{\text{att}}^{(\ell-1)}\cdot \X[\ell-1]\cdot \W^{(\ell)}}
\end{equation}
where each element in the attention adjacency matrix $\bm{A}_{\text{att}}^{(\ell-1)}$ is calculated as:

\begin{equation}
\label{eq: gat attention}
\left[A^{(\ell-1)}_{\text{att}}\right]_{u,v} = \func[LeakyReLU]{\bm{a}^{\trans}\cdot \paren{\bm{W}^{(\ell)}\cdot \bm{x}_u^{(\ell-1)}\Big\|\bm{W}^{(\ell)}\cdot \bm{x}_v^{(\ell-1)}}}
\end{equation}
where $\bm{a}$ is a learnable vector and $\bm{x}_u$ means the feature vector of node $u$ (i.e., the $u$-th row of the feature matrix $\X[\ell-1]$).
As an extension, Equation \ref{eq: gat forward} can be modified to support ``multi-head'' attention. Note that the computation pattern of ``multi-head'' GAT is the same as that of ``single-head'' captured by Equation \ref{eq: gat forward} and our parallelization strategy can be easily extended to support the multi-head version. We therefore restrict to Equation \ref{eq: gat forward} for the discussion on GAT. 
}

{In summary, considering all the four models,} the full forward propagation during training takes $\X$ as the input and generates $\mathcal{L}$ as the output by traversing the GNN layers, the classifier layers, and the loss layer. After obtaining $\mathcal{L}$, we perform backward propagation from the loss layer all the way to the first GNN layer and update the weights by gradients. The gradients are computed by chain-rule. In Section \ref{sec: para spmm}, we analyze the computation in backward propagation and propose parallelization techniques for each of the key operations.

\subsection{Minibatch Training Methods}
\label{sec: background minibatch}

For large scale graphs, training of the {GNN} has to proceed in minibatches, so that each iteration of weight update involves only a small number of graph nodes. 
GraphSAGE \citep{graphsage}, FastGCN \citep{fastgcn}, AS-GCN \citep{asgcn} and S-GCN \citep{sgcn} incorporate various \emph{layer sampling} techniques to construct minibatches. Upper part of Figure \ref{fig: illustration} abstracts the meta-steps of 
\begin{enumerate*}
\item Constructing a full GNN on the training graph $\G$,
\item Sampling nodes from the $\size{\V}$ nodes of each layer, and
\item Forward and backward propagation among the sampled nodes. 
\end{enumerate*}
For the sampling of step 2, various techniques have been proposed to improve learning quality or training speed. 
{For \citep{graphsage,asgcn,sgcn}, they first randomly select a small number of nodes from the outputs of the last GNN layer as the ``minibatch'' nodes. Then they treat such minibatch nodes as the roots and back-tracks the layer connections to sample connected nodes in the previous layers. When such back-tracking goes from layer $L$'s outputs down to layer $1$'s inputs, the number of multi-hop neighbors of the roots can be orders of magnitude larger than the number of roots. This is referred to as ``neighbor explosion'' \citep{graphsage, sgcn, fastgcn} (see also analysis in Section \ref{sec: gcn efficiency}). Note that if $u$ is a $k$-hop neighbor of $v$, then $u$ is connected to $v$ via a length-$k$ path in $\G$. Equivalently, node $u$ in layer $\ell$ of the GNN can influence $v$ in layer $\ell+k$. }
While \citep{asgcn,sgcn} have proposed techniques to alleviate such ``neighbor explosion'' of \citep{graphsage}, none of them is scalability from the computation complexity perspective. Specifically,
the variance reduction based sampler of \citep{sgcn} comes at the cost of much higher memory usage, and the sampler of \citep{asgcn} using an auxiliary neural network incurs significant computation overhead. 
On the other hand, for \citep{fastgcn}, the sampling is performed independently at each layer. \citep{fastgcn} first computes the sampling probability for each node of $\V$, based on the sparse adjacency matrix $\bm{A}$. 
Then it selects a fixed number of nodes from each layer according to such probability. Finally, {the sampled GNN to generate the embedding for the minibatch} is built by connecting the sampled nodes in adjacent layers. 
Clearly, \citep{fastgcn} avoids ``neighbor explosion'' since the number of samples in each layer is fixed. Unfortunately, such training can result in significant accuracy degradation. Since the sampling in each layer is independent, significant portion of the node samples in layer $i$ may not have connection to node samples in layer $i+1$ when $\G$ is large. 

In our prior work \citep{ipdps}, we proposed a minibatch training method for the GraphSAGE model based on graph sampling, and developed parallelization strategies targeting at shared-memory multi-core processors. 
We designed a table based data structure to support parallel graph sampling, and a data partitioning scheme supporting parallel feature propagation within subgraphs. 
In this work, we improve the parallel graph sampling algorithm by a more compact design of the data structure. 
Thus, we significantly reduce the computation cost and storage overhead of graph sampling. 
We also propose a scheduling algorithm for the overall training. The scheduler intelligently re-orders the operations in GNN layer propagation to reduce computation complexity, and updates the sampled subgraphs to improve the cache performance. 
Lastly, we show that our parallelization and scheduling strategies are general, and can be extended to various GNN models including but not limited to GraphSAGE.

Our other work, GraphSAINT \citep{graphsaint}, extends the idea of training GNNs with graph sampling. GraphSAINT focuses on further improving training accuracy by bias elimination and variance reduction techniques, while this work mostly focuses on the parallelization strategies to achieve superior scalability on multi-core platforms. 
Note that the training algorithm enhancements proposed by GraphSAINT can be easily incorporated into our parallel execution framework without losing any efficiency or scalability.

% Since layer sampling requires each node in the graph convolution layer $\ell$ to select multiple neighbor nodes in layer $\ell-1$, the deeper the layer sampler goes, the more nodes will be sampled (i.e., $\size{\Vls^{(\ell)}}$ becomes larger when $\ell$ gets smaller). We refer to this phenomenon as ``neighbor explosion''. % and analyze its inefficiency in Section \ref{sec: gcn efficiency}.
% On the other hand, \citep{fastgcn} proposes a node based layer sampler. It samples in a two phased fashion. The first phase samples nodes of all the $L$ layers based on a pre-calculated probability distribution, where the samplers of each layer are independent. The second phase re-constructs the inter-layer edges to connect the sampled nodes. Empirically, this method mitigates the ``neighbor explosion'' problem  at the cost of accuracy loss and potentially expensive pre-processing (calculating the probability distribution over the entire $\V$). Deeper layers may still require larger sampling population to avoid overly sparse inter-layer connection. 

% CNN: projection of 3D objects onto 2D planes. 
% GCN: projection of original graph onto subgraphs. 

\section{Graph Sampling-Based Minibatch Training}

%Although layer sampling techniques have become popular to enable mini-batched GCN training, state-of-the-art GCN models based on such design philosophy does not simultaneously satisfy the requirements of accuracy, efficiency and scalability. 
We present a novel graph sampling-based GNN training method. Our parallel minibatch training simultaneously outperforms the state-of-the-art in accuracy, efficiency and scalability. We present the design of the graph sampling-based minibatch training (Section \ref{sec: gcn design}), and analyze the advantages in efficiency (Section \ref{sec: gcn efficiency}) and accuracy (Section \ref{sec: gcn accuracy}). We then present optimizations to scale training on parallel machines (Sections \ref{sec: para sample} and \ref{sec: para spmm}).

\subsection{Design of the Minibatch Training Algorithm}
\label{sec: gcn design}

As shown in the lower part of Figure \ref{fig: illustration}, the graph sampling-based approach does not construct a GNN directly on the original input graph $\G$. Instead, for each iteration of weight update during training, we first sample a small induced subgraph $\G[s]\paren{\V[s],\E[s]}$ from $\G(\V,\E)$. 
We then construct a {\emph{complete}}{\footnote{Not to be confused with ``complete graph''. Here a GNN being complete means that the bi-adjacency matrix defining the GNN inter-layer connection has the same non-zeros as the adjacency matrix of the graph $\G[s]$. i.e., we don't perform any sampling on the nodes in each GNN layer or the edges connecting consecutive layers. }} GNN on $\G[s]$. The forward and backward propagation are both on this small GNN. Algorithm \ref{algo: gsaint training} describes our approach. The key distinction from traditional training methods is that the computations (lines 5-13) are performed on nodes of the sampled graph instead of the sampled layer nodes, {thus requiring much less computation in training due to reduced redundancy (Section \ref{sec: gcn efficiency}). 
In addition, since the GNN on the subgraph $\G[s]$ is complete, the forward propagation rule is almost the same as that of the GNN on the full graph.} 
{We can directly use Equations \ref{eq: graphsage forward}, \ref{eq: gcn forward}, \ref{eq: mixhop forward}, \ref{eq: gat forward}, \ref{eq: mlp forward} and \ref{eq: loss forward} by just replacing the full feature matrix $\X[\ell]$ and the full adjacency matrix $\bm{A}$ with the ones for the subgraph, $\X[\ell]_s$ and $\bm{A}_s$. }
In Section \ref{sec: gcn accuracy}, 
we discuss the requirements for the $\sample$ function (line 3), and present three representative graph samplers that leads to high accuracy of training.

{Note that for all the methods discussed in this paper (both the layer sampling based and our proposed graph sampling based), a ``\emph{minibatch}'' is always defined as node samples in the output GNN layer. 
For example, consider a GNN with one hidden layer. If a particular method selects 1000, 100 and 10 nodes in the input, hidden and output layers respectively, then we say the \emph{minibatch size} is 10, the \emph{1-hop neighborhood size} is 100 and the \emph{2-hop neighborhood size} is 1000. In this case, the GNN only generates label predictions for the 10 minibatch nodes. The number of hops is with respect to minibatch nodes. }

\begin{algorithm}
\caption{Graph sampling based minibatch training algorithm}
\label{algo: gsaint training}
\begin{algorithmic}[1]
\renewcommand{\algorithmicrequire}{\textbf{Input:}}
\renewcommand{\algorithmicensure}{\textbf{Output:}}
\Require Training graph $\G(\V,\E,\X)$; Ground-truth labels $\overline{\bm{Y}}$; $L$-layer GCN model
\Ensure GNN with trained weights
%\State Construct $L$-layer GCN to get $\Set{\V^{(\ell)}}$, $\Set{\E^{(\ell)}}$
\LineComment{\color{blue}Iterate over minibatches\color{black}}
\While{not converged}
    \State $\G[s]\paren{\V[s],\E[s]}\gets \sample[\G\paren{\V,\E}]$
    \State $\An_s\gets$ adjacency matrix of $\G[s]$
    \State $\X_s\gets$ minibatch feature matrix by looking up $\X$ with $\V[s]$
    \State $\overline{\bm{Y}}_s\gets$ minibatch ground-truth labels by looking up $\overline{\bm{Y}}$ with $\V[s]$
    \State Construct complete GNN on $\G[s]$
    \LineComment{\color{blue}Forward propagation (with GraphSAGE model as an example)\color{black}}
    \For{$\ell=1$ to $L$}%\color{blue}\Comment{Forward propagation}\color{black}
        \State $\X[\ell]_s\gets\func[ReLU]{\An\cdot \X[\ell-1]_s\cdot \Wn[\ell]\bigg\|\X[\ell-1]_s\cdot \Ws[\ell]}$
    \EndFor
    %\LineComment{\color{blue}Label prediction by embedding; SGD weight update\color{black}}
    \State $\bm{Y}_s\gets \sigma\paren{\func[ReLU]{\X[L]_s\cdot \W[MLP]}}$
    \State $\mathcal{L}_s\gets \func[CE]{\bm{Y}_s,\overline{\bm{Y}}_s}$
    \LineComment{\color{blue}Backward propagation\color{black}}
    \State Update weights $\W[MLP]$, $\Ws[\ell]$, $\Wn[\ell]$ by gradients with respect to $\mathcal{L}_s$
    %\State $\text{ADAM}\Big(C,\bm{X},\Set{\Hgs^{(\ell)}},\Set{\Wneigh^{(\ell)}}, \Set{\Wself^{(\ell)}}\Big)$
\EndWhile
%\STATE $\Vls^{(L)}\gets $ Sampled nodes in $\V^{(L)}$ (uniformly at random)\\
%\State \Return $\Set{\Wself^{(\ell)}},\Set{\Wneigh^{(\ell)}}$
\State \Return Trained GNN model
\end{algorithmic} 
\end{algorithm}

\subsection{Complexity of Graph Sampling-Based Minibatch Training}
\label{sec: gcn efficiency}

We analyze the computation complexity of our graph-sampling based training and show that it {significantly reduces redundancy in computation}. In the following analysis, we do not consider the sampling overhead, and we only focus on the forward propagation, since backward propagation has identical computation characteristics as forward propagation. Later, we also experimentally demonstrate that our technique is significantly faster even with the sampling step included (see Section~\ref{sec: exp}).

{Using the GraphSAGE design as a representative GNN model (Equation \ref{eq: graphsage forward}), the main operations to propagate forward by one GNN layer include:}
\begin{itemize}
    \item \emph{Feature aggregation}: Each node feature vector from layer-$\ell$ propagates via layer connections. The aggregation requires $\mathcal{O}\paren{\size{\E[s]}\cdot \f[\ell]}$ operations.
 \item \emph{Weight transformation}: Each node multiplies its feature with the weight, leading to the overall complexity of
$\mathcal{O}\paren{\size{\V[s]}\cdot \f[\ell-1] \cdot \f[\ell]}$.
\end{itemize}

For simplicity, assume $\f[\ell]=f$. Further let $d_s$ be the average degree of the subgraph $\G[s]$. Complexity of $L$-layer forward propagation in one minibatch is:

\begin{align}
\label{eq: Dcmp batch}
\mathcal{O}\paren{L\cdot \size{\V[s]}\cdot \f\cdot \paren{\f+d_s}}\\
\end{align}

By convention, one epoch of training is defined as one time traversal of all the training data points by predicting their labels. 
Thus, {by the definition of ``minibatch'' in Section \ref{sec: gcn design}}, we define an epoch in our training as $\size{\V}/\size{\V[s]}$ number of minibatches (i.e., subgraphs). Clearly, the computation complexity of an epoch is $\mathcal{O}\paren{L\cdot \size{\V}\cdot f\cdot (f+d_s)}$.

\paragraph*{Comparison Against Other {GNN} Training Methods}
As discussed in Section \ref{sec: background minibatch}, for \citep{graphsage, sgcn}, each sampled node in layer $\ell$ further selects $d'$ number of neighbors in layer $\ell-1$. For \citep{graphsage}, $d'$ ranges from $10$ to $50$, and for \citep{sgcn}, $d'=2$. 
%a certain number $d_\text{LS}$ of neighbors are selected from the next layer for each sampled node in the current layer. 
So depending on the {minibatch size (see Section \ref{sec: gcn design})}, the complexity of one epoch falls between: 

\textit{Case 1 [Small {minibatch} size]:} $\mathcal{O}\paren{\paren{d'}^{L}
\cdot \size{\V}\cdot f\cdot(f+d')}$.

\textit{Case 2 [Large {minibatch} size]}
$\mathcal{O}\paren{L\cdot \size{\V}\cdot f\cdot(f+d')}$.

We observe that when the {minibatch} size is much smaller than the training graph size, the layer sampling techniques result in high training complexity (computation load grows exponentially with {GNN} depth). 
Essentially, due to ``neighbor explosion'', when the layer-$L$ nodes are traversed only once, the nodes in the previous layer $\ell$ are sampled and evaluated $\paren{d'}^{L-\ell}$  times on average. 
{The repeated evaluation of the layer nodes across different minibatches makes training inefficient due to computation redundancy}. 
On the other hand, when the {minibatch} size of \citep{graphsage,sgcn} becomes comparable to the training graph size, the training complexity grows linearly with the {GNN} depth and training graph size. However, the resolution of ``neighbor explosion'' comes at the cost of slow convergence and low accuracy \citep{large_batch_size}, {since overly large minibatch size hurts generalization of neural networks}. So such training configuration {of Case 2} does not scale to large graphs.

{If we ignore the convergence rate dependent on the input graph, our graph-sampling based training leads to a parallel algorithm whose complexity is linear in GNN depth and training graph size. The work-efficiency of our training is guaranteed by design: throughout the entire training, for each node $v$, the number of times its label is predicted in the output layer is equal to the number of times its feature is computed in any hidden layer. In this sense, there is no redundant computation arising from repeated evaluation of hidden layer nodes as discussed above.} In addition, by choosing proper graph sampling algorithms, we can construct small representative subgraphs whose sizes do not {grow proportionally with} the training graph size (as shown in Section \ref{sec: exp}). 

\subsection{Accuracy of Graph Sampling-Based Training}
\label{sec: gcn accuracy}

%From intuition, the GCN model naturally works with sampled graphs. Any graph (even the full training graph $\G$) can be viewed as a sampled graph from a possibly infinitely large graph $\G'$ in the world. A GCN built upon $\G$ is an estimate on the ideal GCN built upon $\G'$. GCNs built upon sampled subgraphs $\Gsub$ are estimates on the GCN built upon $\G$. If the sampling algorithm constructs enough number of representative subgraphs $\Gsub$, our graph sampling-based GCN model can absorb all the information in $\G$, and generate embedding without accuracy loss. 
Layer-based sampling methods assume that a subset of neighbors of a given node is sufficient to learn its representation. We achieve the same goal by sampling the graph itself. If the sampling algorithm constructs enough number of representative subgraphs $\G[s]$, our training process should absorb all the information in $\G$, and generate accurate embeddings.
More specifically,
as discussed in Section \ref{sec: gcn background}, the output vectors ``embed" the input graph topology as well as the initial node attributes. A good graph sampler, thus, should guarantee:

\begin{enumerate}
    \item Sampled subgraphs preserve the connectivity characteristics of the training graph. 
    \item Each training graph node has non-negligible probability to be sampled.
\end{enumerate}

%There are well-established metrics to quantify the connectivity characteristics of a graph, including cluster coefficients, degree distribution, vertex label density and so on. We choose the well-known frontier sampling algorithm to sample $\Gsub$ for our GCN model. 

{It has been widely studied \citep{sampling_survey} that various random walk based graph sampling algorithms (including unbiased random walk \citep{graphsaint}, forest fire \citep{forest_fire, graph_over_time}, multiple random walk and frontier sampling \citep{frontier}) can preserve the various input graph characteristics well. 
In addition, all these sampling algorithms are able to explore the full set of nodes and edges in the original graph due to the stochasticity in sampling. 
Thus, such algorithms are all valid candidates for our subgraph sampling based training. 
From the perspective of computation, unbiased random walk, forest fire and multiple random walk algorithms fall within the ``\emph{static}'' category of the random walk family according to \citep{knightking}. In other words, throughout the sampling process, these three sampling algorithms follow a fixed probability distribution on node or edges, regardless of the historically traversed subgraph structure. 
However, the frontier sampling algorithm maintains a \emph{dynamic} probability distribution updated by the ``frontier nodes'' at the current timestamp. 
Therefore, for frontier sampling, computation complexity as well as difficulty in parallelization are both higher compared with the other three static algorithms. 
In the following, we use frontier sampling as a representative and analyze in detail its performance in terms of accuracy and parallel execution. 
We then discuss how the proposed techniques can be extended to the other three samplers in Section \ref{sec: sampler extension}. 
}

{Before going into the specific steps in sampling, we first give some intuition on why training with frontier sampling may lead to high accuracy. Recall the two requirements above characterizing a good sampler. }
For requirement 1, while ``connectivity'' may have several definitions, subgraphs output by \citep{frontier} approximate the original graph with respect to multiple connectivity measures, {including degree distribution, assortative mixing coefficient and clustering coefficients. These graph measures critically define how signals on the graph nodes would propagate and mix via GNN layers, and thus should be carefully maintained by the subgraph samples. } For requirement 2, during initialization, the frontier sampler picks some root nodes uniformly at random from the original graph {(see Section \ref{sec: para sample baseline})}. These roots constitute a significant portion of the subgraph nodes. Thus, over large enough number of sampling iterations, all input attributes of the training graph will be covered by the frontier sampler. 
For readers interested in theoretical justification on the choice of those sampling algorithms, please check the analysis in \citep{graphsaint}.

%From the high-level, our approach is sampling on the training graph, and GraphSAGE / FastGCN is sampling on the neural network. 

%\emph{Benefit on work-efficiency}: List the computation complexity and communication volume equation.

\section{Parallel Graph Sampling Algorithm}
\label{sec: para sample}

{In this section , we first describe in detail our parallelization strategies for the frontier sampling algorithm \citep{frontier}. Then in Section \ref{sec: sampler extension}, we show how to extend our strategies to other graph samplers. }

\subsection{Graph Sampling Algorithm}
\label{sec: para sample baseline}

The frontier sampling algorithm proceeds as follows. 
Throughout the sampling process, the sampler maintains a constant-size frontier set $\text{FS}$ consisting of $m$ vertices in $\G$. In each iteration, the sampler randomly pops out a node $v$ in $\text{FS}$ according to a degree based probability distribution, and replaces $v$ in $\text{FS}$ with a randomly selected neighbor of $v$. The popped out $v$ is added to the node set $\V[s]$ of $\G[s]$. 
The sampler repeats the above update process on the frontier set $\text{FS}$, until the size of $\V[s]$ reaches the desired budget $n$. Algorithm \ref{algo: frontier sampling} shows the details. According to \citep{frontier}, a good empirical value of $m$ is around $1000$. %Empirically, the authors in \citep{frontier} suggest to set $m=1000$. %For sake of accuracy, we set the budget to $b\in\left[ 2000,8000 \right]$, depending on the training graph. 

\begin{algorithm}
\caption{Frontier sampling algorithm}
\label{algo: frontier sampling}
\begin{algorithmic}[1]
\renewcommand{\algorithmicrequire}{\textbf{Input:}}
\renewcommand{\algorithmicensure}{\textbf{Output:}}
\Require{Training graph $\G(\V,\E)$; Frontier size $m$; Node budget $n$}
\Ensure{Induced subgraph $\G[s]\paren{\V[s],\E[s]}$}
\State $\text{FS}\gets$ Set of $m$ nodes selected uniformly at random from $\V$
\State $\V[s]\gets \text{FS}$
\For{$i=0$ to $n-m-1$}
    \State Select $u\in \text{FS}$ with probability $\func[deg]{u}/\sum_{v\in \text{FS}}\func[deg]{v}$
    \State Select $u'$ from neighbors $\Set{w|(u,w)\in\E}$ uniformly at random
    \State $\text{FS}\gets \paren{\text{FS}\setminus \Set{u}}\cup\Set{u'}$
    \State $\V[s]\gets \V[s]\cup\Set{u}$
\EndFor
\State $\G[s]\gets $ Subgraph of $\G$ induced by $\V[s]$ 
\State \Return $\G[s]\paren{\V[s],\E[s]}$ 
\end{algorithmic} 
\end{algorithm}

In our sequential implementation of training, we notice that about half of the time is spent in the sampling phase. This motivates us to parallelize the graph sampler. The challenges are: 
\begin{enumerate*}
    \item While sampling from a discrete distribution is a well-researched problem, we focus on fast parallel sampling from a {\it dynamic} probability distribution. Such  dynamism is due to the addition/deletion of new nodes in the frontier. Existing methods for fast sampling such as aliasing \citep{alias} (which can output a sample in $\mathcal{O}(1)$ time with linear processing) cannot be modified easily for our problem.
    It is non-trivial to select a node from the evolving $\text{FS}$ with low complexity. %Constructing or updating the probability distribution requires $\mathcal{O}(m)$ work. 
    A straightforward implementation by partitioning the total probability of 1 into $m$ intervals would require $\mathcal{O}\paren{m}$ work to update the intervals for each replacement in FS. Given $m=1000$ as recommended by the authors in the original paper \citep{frontier}, the $\mathcal{O}\paren{m\cdot n}$ complexity to sample a single $\G[s]$ is too expensive. 
    %Empirically, to achieve high accuracy, the graph sampling-based GCN sets $m=\gamma\cdot\abs{\V_{\text{sub}}}$ (where $0.3\leq\gamma\leq 0.5$). This means that the amount of work incurred by a single sampler instance is quadratic to the subgraph size $\abs{\V_{\text{sub}}}$.
    \item The sampling is inherently sequential as the nodes in the frontier set should be popped out one at a time. Otherwise, $\G[s]$ may not preserve the characteristics of the original graph well enough. 
\end{enumerate*}

To address the above challenges, we first propose a novel data structure that lowers the complexity of frontier sampler and allows thread-safe parallelization (Section \ref{sec: para sample intra}). We then propose a training scheduler that exploits parallelization within and across sampler instances (Section \ref{sec: para sample inter} and \ref{sec: scheduler}). 

\subsection{Dashboard Based Implementation}
\label{sec: para sample intra}

Since nodes in the frontier set is replaced only one at a time, an efficient implementation should allow incremental update of the probability distribution over the $m$ nodes. To achieve such goal, we propose a ``Dashboard" table to store the status of current and historical frontier nodes (a node becomes historical after it gets popped out of the frontier set). The next node to pop out is selected by probing the Dashboard using randomly generated indices. 
%The data structure ensures that the probability distribution for picking nodes is updated quickly and incrementally.
%Each probing incurs cost of a simple table lookup. When a vertex $v$ is popped out of the frontier set, some book-keeping operations need to be performed on Dashboard. First of all, the role of $v$ changes from ``current" to ``historical". Secondly, the newly added vertex needs to append its corresponding entries towards the end of Dashboard. 
%Note that such book-keeping is incremental, making the overall complexity low. Also, the probing and book-keeping can be easily parallelized without any data hazards. 
In the following, we formally describe the data structure and operations in the Dashboard-based sampler. The implementation involves two arrays: 

\begin{itemize}
    \item \textbf{Dashboard} $\text{DB}\in\mathbb{R}^{\eta\cdot m\cdot d}$: A vector maintaining the status and sampling probabilities of the current and historical frontier nodes. 
    If a node $v$ is in the frontier, we ``pin'' a ``tile'' of $v$ to the ``dashboard''. 
    Here a tile is a small data structure storing the meta-data of $v$, and a pin is an address pointer to the tile. 
    One entry of DB corresponds to one pin. 
    A node $v$ will have $\func[deg]{v}$ pins allocated continuously in DB, each pointing to the same tile belonging to  $v$. 
    If $v$ is popped out of the frontier, we invalidate all its pins to $\func[NULL]{}$. 
    The optimal value of the parameter $\eta$ is explained later. 

    \item \textbf{Index array} $\text{IA}\in \mathbb{R}^{2\times (\eta\cdot m\cdot d + 1)}$: An auxiliary array to help cleanup $\text{DB}$ upon table overflow. The $j^{\text{th}}$ column in IA has 2 slots, the first slot records the starting index of the DB pins corresponding to $v$, where $v$ is the $j^{\text{th}}$ node added into DB. The second slot is a flag, which is \verb|True| when $v$ is a current frontier node, and \verb|False| when $v$ is a historical one.
\end{itemize}

{
\begin{table}[!ht]
\caption{Summary of symbols related to the Dashboard based frontier sampling}
    \centering
    \begin{tabular}{rl}
        \toprule
        Name & Meaning\\
        \midrule
        \midrule
        \multirow{2}{*}{Dashboard (DB)} & Data structure consisting of ``pins'' and ``tiles'' to support \\
        & fast dynamic update of probability distribution\\
        %\hline
        \cmidrule(lr){2-2}tile & Data structure storing meta-information of frontier nodes\\
        \cmidrule(lr){2-2}\multirow{2}{*}{pin} & Pointer pointing to the tiles. All pins belonging to \\
        & the same node will point to a shared tile\\
        \cmidrule(lr){2-2}Index array (IA) & Data structure helping the cleanup of DB when it is full\\
        \midrule
        $m$ & Number of nodes in the frontier set\\
        \cmidrule(lr){2-2}$n$ & Total number of nodes to be sampled in the subgraph\\
        \cmidrule(lr){2-2}$d$ & Average degree of frontier nodes\\
        \cmidrule(lr){2-2}\multirow{2}{*}{$\eta$} & Enlargement factor controlling the computation-storage\\
        & tradeoff. Larger $\eta$: larger DB and less frequent cleanup\\
        \bottomrule
    \end{tabular}
    \label{tab: dashboard symbols}
\end{table}
}

{The symbols related to the design and analysis of the Dashboard data structure are summarized in Table \ref{tab: dashboard symbols}. }

Since the probability of popping out a node in frontier is proportional to its degree, we allocate $\func[deg]{v_i}$ continuous entries in $\text{DB}$, for each $v_i$ currently in the frontier set. This way, the sampler only needs to probe DB uniformly at random to achieve line 4 of Algorithm \ref{algo: frontier sampling}. 
Clearly, $\text{DB}$ should contain at least $m\cdot d$ entries, where $d$ is the average degree of the frontier nodes. 
%When one of the frontier nodes is replaced by its neighbor, it may be 
%Since the node degrees vary significantly, 
For the sake of incremental updates, we append the entries for the new node and invalidate the entries of the popped out node, instead of changing the values in-place and shifting the tailing entries. The invalidated entries become historical. To accommodate the append operation, we introduce an enlargement factor $\eta$ (where $\eta >1$), and set the length of $\text{DB}$ to be $\eta\cdot m\cdot d$. As an approximation, we set $d$ as the average degree of the training graph $\G$.
%The factor $\eta$ allows $\text{DB}$ to contain some empty (untouched by any vertices) and invalid (occupied by historical frontier vertices, not yet freed) entries, so that we avoid the expensive data shifting operation each time FS and DB get updated.
As the sampling proceeds, eventually, all of the $\eta\cdot m\cdot d$ entries in DB may be filled up by the information of current and historical frontier nodes. In this case, we free up the space occupied by historical nodes before resuming the sampler. Although cleanup of the Dashboard is expensive, due to the factor $\eta$, such scenario does not happen frequently (see complexity analysis in Section~\ref{sec: para sample inter}).
%Lastly, to make the ``cleanup" phase more efficient, we introduce an auxiliary index array IA. For each $v$ in the frontier set FS, the IA array records the starting index of $v$'s corresponding entries in DB. 
Using the information in IA, the cleanup phase does not need to traverse all of the $\eta\cdot m\cdot d$ entries in DB to locate the space to be freed. When DB is full, the entries in DB can correspond to at most $\eta\cdot m\cdot d$ vertices. Thus, we safely set the capacity of IA to be $\eta\cdot m\cdot d$ + 1. Slot 1 of the last entry of IA contains the current number of used DB entries.
%\footnote{It is possible that the frontier set FS contains many high degree vertices, making the average degree of FS vertices larger than $\eta\cdot d$. In this case, DB is simply too small, then we keep doubling $\eta$ until $\eta\cdot m\cdot d$ is large enough. This scenario rarely happens, and we do not consider it in our analysis.}
%\footnote{Explanation on $(m+1)$ factor: If there are $x$ historical and current frontier vertices, then IA contains $x+1$ entries with meaningful values. Last IA entry, slot 1 contains the number of occupied DB entries. Slot 2 is set to \texttt{FALSE}.}

%Digging further into the design details, each of the $\eta\cdot m\cdot d$ DB entry contains 3 slots, and each of the IA entry contains 2 slots, making $\text{DB}\in\mathbb{R}^{3\times(\eta\cdot m\cdot d)}$ and $\text{IA}\in \mathbb{R}^{2\times (\eta\cdot m\cdot d)}$. 
%For a DB entry corresponding to $v_i$, the first slot is $i$; the second slot is an offset value, helping with the invalidation of all of the $v_i$ related entries when $v_i$ is popped out; the third slot is a value $k$, if $v_i$ is the $k^{th}$ vertex added into DB ($k$ counted after the previous cleanup phase). 
%For the $j^{th}$ entry in IA, the first slot records the starting index of the DB entries corresponding to $v$, where $v$ is the $j^{th}$ vertex added into DB. The second slot is a Boolean flag, which is \verb|True| when $v$ is a current frontier vertex, and \verb|False| when $v$ is a historical frontier vertex. 

\subsection{Intra- and Inter-Subgraph Parallelization}
\label{sec: para sample inter}

% ===============================
% data lvl & task lvl parallelism
% ===============================

%The benefits of Dashboard based frontier sampler are that: 
%\begin{enumerate*}
%    \item It supports incremental change in DB when frontier set is updated; and
%    \item Parallel access of DB can be easily supported.
%\end{enumerate*}

%However, the design still has its limitation. Namely, due to the inherent sequential nature of frontier sampling algorithm, it is difficult for a single sampler instance to exploit massive amount of parallelism. 

%This section proposes a hybrid scheme to exploit parallelism within and across multiple sampler instances. We take advantage of the Dashboard data structure, and devise a scheduler in GCN training to scale our sampler to many processing cores. 

Since our subgraph-based GNN training requires independently sampling multiple  subgraphs, we can sample different subgraphs on different processors in parallel. Also, we can further parallelize within each sampling instance by exploiting the parallelism in  probing, book-keeping  and cleanup of DB. %, which is presented next.

\begin{algorithm}[!ht]
\caption{Parallel Dashboard based frontier sampling}
\label{algo: frontier sampling par}
\begin{algorithmic}[1]
\renewcommand{\algorithmicrequire}{\textbf{Input:}}
\renewcommand{\algorithmicensure}{\textbf{Output:}}
\Require{Original graph $\G(\V,\E)$; Frontier size $m$; Budget $n$; Enlargement factor $\eta$; Number of processors $p$}
\Ensure{Induced subgraph $\G[s]\paren{\V[s],\E[s]}$}
\State $d\gets \abs{\E}/\abs{\V}$
\State $\text{DB}\gets$ Array of $\mathbb{R}^{1\times (\eta\cdot m\cdot d)}$ with value \verb|NULL|
\State $\text{IA}\gets$ Array of $\mathbb{R}^{2\times (\eta\cdot m\cdot d + 1)}$ with value \verb|INV|\color{blue}\Comment{INValid}\color{black}
\State $\text{FS}\gets$ Set of $m$ nodes selected uniformly at random from $\V$
\State $\V[s]\gets \text{FS}$
\State Convert the set FS to an indexable list of nodes
\State $\text{IA}\left[ 0,0 \right]\gets 0;\qquad\text{IA}\left[ 1,0 \right]\gets $\verb|True|$;$
\For{$i=1$ to $m$}\color{blue}\Comment{Initialize IA from FS}\color{black}
    \State $\text{IA}\left[ 0,i \right]\gets \text{IA}\left[ 0,i-1 \right]+\func[deg]{\text{FS}\left[ i-1 \right]}$
    \State $\text{IA}\left[ 1,i \right]\gets \texttt{True}$
\EndFor
\State $\text{IA}\left[ 1,m \right]\gets \texttt{False}$

\ParFor{$i=0$ to $m-1$}\color{blue}\Comment{Initialize DB from FS}\color{black}
    \State $\text{pin}\gets $ Address of a tile of 4-tuple $\paren{\text{FS}[i],\text{IA}[0,i],\text{IA}[0,i+1],i}$
    \For{$k=\text{IA}\left[ 0,i \right]$ to $\text{IA}\left[ 0,i+1 \right]-1$}
        \State $\text{DB}[k]\gets $ pin
    \EndFor
\EndParFor

\State $\text{cnt}\gets m;\qquad \V[s]\gets \emptyset;$
\For{$i=m$ to $n-1$}\color{blue}\Comment{Main loop of sampling}\color{black}
    \State $v_{\text{pop}}, \text{pin}\gets \func[pardo\_POP\_FRONTIER]{\text{DB},p}$
    \State $v_{\text{new}}\gets $ Node randomly sampled from $v_{\text{pop}}$'s neighbors
    \If{$\func[deg]{v_{\text{new}}}>\eta\cdot m\cdot d-\text{IA}\left[ 0,s \right]+1$}
        \State $\text{DB}\gets \func[pardo\_CLEANUP]{\text{DB},\text{IA},p}$
        \State $\text{cnt}\gets m-1$
    \EndIf
    \State $\func[pardo\_ADD\_TO\_FRONTIER]{v_{\text{new}},\text{pin},\text{cnt},\text{DB},\text{IA},p}$
    \State $\V[s]\gets \V[s]\cup\Set{v_{\text{new}}}$
    \State $\text{cnt}\gets \text{cnt}+1$
\EndFor

\State $\G[s]\gets $ Subgraph of $\G$ induced by $\V[s]$ 
\State \Return $\G[s]\paren{\V[s],\E[s]}$ 
\end{algorithmic} 
\end{algorithm}

\begin{algorithm}[!ht]
\caption{Functions in Dashboard Based Sampler}
\label{algo: frontier sampling func}
\begin{algorithmic}[1]
\Function{pardo\_POP\_FRONTIER}{$\text{DB},p$}
\State $\text{idx}_{\text{pop}}\gets $\verb|INV|\color{blue}\Comment{Shared variable}\color{black}
\ParFor{$j=0$ to $p-1$}
    \While{$\text{idx}_{\text{pop}}==$ \texttt{INV}}\color{blue}\Comment{Probing DB}\color{black}
        \State $\text{idx}_p\gets $ Index generated uniformly at random
        \If {$\text{DB}\left[ \text{idx}_p \right]\neq\texttt{NULL}$}
            \State $\text{idx}_{\text{pop}}\gets \text{idx}_p$            
        \EndIf
    \EndWhile
\EndParFor
\State $\text{pin}_\text{pop}\gets \text{DB}\left[\text{idx}_\text{pop}\right]$
\State $v_{\text{pop}},i_\text{pinStart},i_\text{pinEnd},i_\text{IA}\gets \text{data of the tile pointed to by }\text{pin}_\text{pop}$
    
\ParFor{$j=0$ to $p-1$}
    \State Update the DB entries to $\func[NULL]{}$ from index $i_\text{pinStart}$ to $i_\text{pinEnd}$ 
\EndParFor
\State $\text{IA}\left[ 1,i_\text{IA} \right]\gets \texttt{False}$\color{blue}\Comment{Update IA}\color{black}
\State \Return $v_{\text{pop}}$, $\text{pin}_\text{pop}$
\EndFunction

\Function{pardo\_CLEANUP}{$\text{DB},\text{IA}$,$p$}
    \State $\text{DB}_\text{new}\gets $ New, empty dashboard
    \State $k\gets$ Cumulative sum of $\text{IA}\left[ 0,: \right]$ masked by $\text{IA}\left[ 1,: \right]$
    \ParFor {$i=0$ to $p-1$}
        \State Move entries from DB to $\text{DB}_\text{new}$ by offsets in $k$
    \EndParFor
    \ParFor {$i=0$ to $p-1$}
        \State Re-index IA based on $\text{DB}_\text{new}$    
    \EndParFor
    \State \Return $\text{DB}_\text{new}$
\EndFunction

\Function{pardo\_ADD\_TO\_FRONTIER}{$v_\text{new},\text{pin},i,\text{DB},\text{IA},p$}

\State $\text{IA}\left[ 0,i+1 \right]\gets \text{IA}\left[ 0,i \right]+\func[deg]{v_\text{new}};\qquad \text{IA}\left[ 1,i \right]\gets \texttt{True};$
\State Assign values $\paren{v_\text{new},\text{IA}[0,i],\text{IA}[0,i+1],i}$ to the tuple pointed to by pin
\ParFor{$j=0$ to $p-1$}
    \State Update the DB entries to pin from index $\text{IA}[0,i]$ to $\text{IA}[0,i+1]$ 
\EndParFor

\EndFunction
\end{algorithmic}
\end{algorithm}

Algorithm \ref{algo: frontier sampling par} shows the details of Dashboard-based parallel frontier sampling, where all arrays are zero-based. Considering the main loop (lines 20 to 30), we analyze the complexity of the three functions in Algorithm \ref{algo: frontier sampling func}. Denote $\text{COST}_{\text{rand}}$ and $\text{COST}_{\text{mem}}$ as the cost to generate one random number and to perform one memory access, respectively. 

\paragraph*{pardo\_POP\_FRONTIER}

Anytime during sampling, on average, the ratio of valid DB entries (those occupied by current frontier vertices) over total number of DB entries is $1/\eta$. Probability of one probing falling on a valid entry equals $1/\eta$. 
%Probability of at least one of $p$ probings falling on a valid entry equals $1-\big(1-\frac{1}{\alpha}\big)^p$.
Expected number of rounds for $p$ processors to generate at least 1 valid probing can be shown to be $1/\paren{1-\paren{1-\frac{1}{\eta}}^p}$, where one round refers to one repetition of lines 5 to 7 of Algorithm \ref{algo: frontier sampling func}. After selection of $v_{\text{pop}}$, $\func[deg]{v_{\text{pop}}}$ number of slots needs to be updated to invalid values \verb|INV|. 
Since this operation occurs $(n-m)$ times, the $\func[para\_POP\_FRONTIER]{}$ function incurs $(n-m) \paren{\frac{1}{1-(1-1/\eta)^p}\cdot \text{COST}_{\text{rand}}+\frac{d}{p}\cdot \text{COST}_{\text{mem}}}$ cost.

\paragraph*{pardo\_CLEANUP}

Each time cleanup of DB happens, we need one traversal of IA to calculate the cumulative sum of indices (slot 1) masked by the status (slot 2), so as to obtain the new location for each valid entries in DB. On expectation, only $\eta\cdot m$ entries of IA is filled, so this step costs $\eta\cdot m$. Afterwards, only the valid entries in DB will be moved to the new, empty DB based on the accumulated shift amount. This translates to $m\cdot d$ number of memory operations. The $\func[para\_CLEANUP]{}$ function is fully parallelized. %For ease of analysis, we ignore the complexity of masked cumulative sum on IA (as shown by analysis later on, $\eta\ll 3\cdot d$). 
The cleanup happens only when DB is full, i.e., $\frac{n-m}{(\eta-1)m}$ times throughout sampling. Thus, the cost is $\frac{n-m}{(\eta-1)\cdot m}\cdot \frac{m\cdot d}{p}\cdot \text{COST}_{\text{mem}}$. We ignore the cost of computing the cumulative sum as $\eta m\ll md$.

\paragraph*{pardo\_ADD\_TO\_FRONTIER}

Adding a new frontier $v_\text{new}$ to DB requires appending $\func[deg]{v_\text{new}}$ new entries to DB. This costs $(n-m)\cdot \frac{d}{p}\cdot \text{COST}_{\text{mem}}$.

{Considering all operations in pardo\_POP\_FRONTIER, pardo\_CLEANUP and pardo\_ADD\_TO\_FRONTIER}, the overall cost to sample one subgraph on $p$ processors equals:

\begin{align}
\label{eq: frontier single}
\left(\frac{1}{1-(1-1/\eta)^p}\cdot \text{COST}_{\text{rand}} +\left( 2 + \frac{1}{\eta -1}\right)\frac{d}{p}\cdot\text{COST}_{\text{mem}}\right)\cdot (n-m)
\end{align}

Assuming $\text{COST}_{\text{mem}} = \text{COST}_{\text{rand}}$, we have the following scalability bound:

\begin{thm}
\label{thm: frontier scale}
For any given $\epsilon > 0$, Algorithm~\ref{algo: frontier sampling} guarantees a speedup of at least $\frac{p}{1+\epsilon}, \forall p\leq \epsilon d\paren{2+\frac{1}{\eta-1}} -\eta$. 
\end{thm}
\begin{proof}
Note that $\frac{1}{1-(1-1/\eta)^p} \leq \frac{1}{1-\func[exp]{-p/\eta}} \leq  \frac{\eta + p}{p}$. This follows from $\frac{1}{1-e^{-x}} = \frac{1}{1-\frac{1}{e^x}} \leq \frac{1}{1-\frac{1}{1+x}} \leq \frac{x+1}{x}$. Further, since $p \leq \epsilon d\cdot\paren{2+1/(\eta-1)} - \eta$, we have $\frac{\eta + p}{p} \leq \frac{\epsilon d\cdot(2+1/(\eta-1))}{p}$.
% We assume that $p\geq 2$. Therefore, 
% \begin{align*}
%     \frac{p}{1-\left(1 - 1/\eta\right)^p} \leq  \frac{p}{1-\left(1 - 1/\eta\right)^2} \leq \frac{\frac{\epsilon d}{\eta}\left(1-\frac{1}{2\eta}\right)}{1-\left(1 - 1/\eta\right)^2}.
% \end{align*}
%Therefore, we have $\frac{1}{1-\left(1 - 1/\eta\right)^p} \leq \frac{4\epsilon d}{p}$.
Now, speedup obtained by Algorithm~\ref{algo: frontier sampling} compared to a serial implementation ($p=1$) is
\begin{align*}
    &\frac{\paren{\eta + d(1/(\eta-1) + 2)}(n-m)}{\left(\frac{1}{1-(1-1/\eta)^p} + \frac{d}{p}(1/(\eta-1) + 2)\right)(n-m)}\\
    &\geq \frac{d(1/(\eta-1) + 2)}{\frac{\epsilon d}{p}(1/(\eta-1) + 2) + \frac{d}{p}(1/(\eta-1) + 2)} \geq \frac{p}{1+\epsilon}.
\end{align*}
\end{proof}

Setting $\epsilon=0.5$, {then for any value of $\eta$}, Theorem \ref{thm: frontier scale} guarantees good scalability ($p/1.5$) for at least {$p=d-\eta$ processors. As we will see later in this section, we perform the intra-sampler parallelism via AVX instructions. So we do not require $p$ to scale to a large number in practice. }
{Note that the above performance analysis always holds as long as we know the expected node degree in the subgraphs. 
During the sampling process, when the sampler enters a well connected local region of the original graph, cleanup may happen more frequently since the frontier contains more high degree nodes. However, the sampler would eventually replace those high degree frontier nodes with low degree ones, so that the overall subgraph degree is similar to that of the original graph. 
Also, note that for graphs with skewed degree distribution, it is possible that the next node to be added into the frontier set has very high degree. Such a node may even require more slots than that is totally available in DB. In this case, we would cleanup DB and allocate all the remaining slots to that node, without further expanding the size of DB. This only slightly alters the sampling distribution since the higher the node degree is, the sooner it would be popped out of the frontier. In the experiments, we also obverse that such a corner case does not affect the training accuracy (see Section \ref{sec: exp acc}). }

While the scalability can be high for dense graphs, it is challenging to scale the sampler to massive number of processors on sparse graphs. Feasible parallelism is bound by the graph degree. In summary, the parallel Dashboard based frontier sampling algorithm
\begin{enumerate*}
\item enables lower serial complexity by incremental update on probability distribution, and
\item  scales well up to $p = \mathcal{O}(d)$ number of processors.
\end{enumerate*}
{Compared with our original Dashboard based sampling in \citep{ipdps}, the data structure presented in this section is more compact. In the original design, the meta-data of a frontier node $v$ (i.e., the 4-tuple in line 14 of Algorithm \ref{algo: frontier sampling par}) is repeatedly stored $\func[deg]{v}$ times in DB. In the current design, the meta data is only stored once by introducing the ``pin-tile'' mechanism. Thus, the DB size is reduced from $4\cdot \eta\cdot m\cdot d$ to $\eta\cdot m\cdot d$. Such ``pin-tile'' design significantly reduces both the memory storage and the memory movement cost simultaneously. }

%We observe the following: 
%\begin{enumerate*}
%    \item Amount of parallelism $p$ is not very large (no more than the average degree $d$); and
%    \item For small number of processors $p$, the total workload of $p$ processors $p\cdot \text{COST}_{\text{samp}}\approx \big(\eta \cdot \text{COST}_{\text{rand}}+\frac{3\cdot d}{\eta}\cdot \text{COST}_{\text{mem}}+\text{constant}\big)\cdot (b-m)$ admits a minimizer $\eta=\sqrt{3\cdot d\cdot \frac{\text{COST}_{\text{mem}}}{\text{COST}_{\text{rand}}}}$. The above analysis uses binomial approximation $(1+x)^{\alpha}\approx 1+\alpha\cdot x$
%\end{enumerate*}

To further scale the graph sampling step, we exploit task parallelism across multiple sampler instances. Since the topology of the training graph $\G$ is fixed over the training iterations, sampling and GNN computation can proceed in an interleaved fashion, without any dependency constraints. 
Detailed scheduling algorithm of the sampling phase and the GNN computation phase is described in Section \ref{sec: scheduler}. 
The general idea is that, during training, we maintain a pool of sampled subgraphs $\Set{\G[i]}$. When $\Set{\G[i]}$ is empty, the scheduler launches $p_{\text{inter}}$ frontier samplers in parallel, and fill the pool with subgraphs independently sampled from the full graph $\G$. Each of the $p_{\text{inter}}$ sampler instances runs on $p_{\text{intra}}$ number of processing units. Thus, the scheduler exploits both intra- and inter-subgraph parallelism. In each training iteration, we remove a subgraph $\G[s]$ from $\Set{\G[i]}$, and build a complete GNN upon $\G[s]$. Forward and backward propagation stay the same as lines 9 to 15 in Algorithm \ref{algo: gsaint training}.

%Tradeoff exists between the 

When filling the pool of subgraphs, total amount of parallelism $p_{\text{intra}}\cdot p_{\text{inter}}$ is fixed on the target platform. We should choose the value of $p_\text{intra}$ and $p_\text{inter}$ carefully chosen based on the trade-off between the two levels of parallelism. Note that the operations on DB mostly involve a chunk of memory with continuous addresses. This indicates that intra-subgraph parallelism can be well exploited at the instruction level using vector instructions (e.g., AVX). In addition, since most of the memory traffic going into DB is in a random manner, it is desirable to have DB stored in cache. 
As coarse estimation, with $m=1000$, $\eta=2$, $d=25$, the memory consumption by one DB is $400\text{KB}${\footnote{Assume 8-byte address pointing to the tuple of pins. So the size of DB is $2\cdot 1000\cdot 25\cdot 8$ Bytes and the size for the pins is $1000\cdot 4\cdot 4$ Bytes.}}. This indicates that DB mostly fits into the private L2 cache (size $256\text{KB}$) in modern shared memory parallel machines. Therefore, during sampling, we bind one sampler to one processor core, and use AVX instructions to parallelize within a single sampler. For example, on a 40-core machine with AVX2, $p_\text{intra}=8$ and $p_\text{inter}=40$. 

{Finally, note that the size of DB is determined by the number of frontier nodes, $m$, rather than the number of subgraph nodes $n$. 
While it is true that we may need to increase $n$ when the original training graph $\G$ grows, the size of $m$ would not need to change. 
The authors of \citep{frontier} interpret $m$ as the dimensionality of the random walk --- frontier sampling on $\G$ is equivalent to a single random walk on $\G$ raised to the $m$-th Cartesian power. 
With such understanding, the authors of \citep{frontier} use a fixed number of $m=1000$ on all experiments in ranging from small graphs to large ones. 

\subsection{Extension to Other Graph Sampling Algorithms}
\label{sec: sampler extension}

By Section \ref{sec: gcn accuracy}, it is reasonable to use other graph sampling algorithms to perform minibatch GNN training. 
Here we evaluate two sampling algorithms: random edge sampling (``Edge'') and unbiased random walk sampling (``RW''). The two algorithms are recommended in \citep{graphsaint}. The ``Edge'' sampler assigns the probability of picking an edge $(u,v)$ as $p_{u,v}\propto \frac{1}{\func[deg]{u}}+\frac{1}{\func[deg]{v}}$, and can be understood as a special case of the ``RW'' algorithm by setting the walk length to be 1. Algorithm \ref{algo: other sampler} specifies the steps of the two algorithms. 
Under the categorization in Section \ref{sec: gcn accuracy}, ``Edge'' and ``RW'' samplers are static since the sampling probability does not change during the sampling process. Therefore, their computation complexity is much lower than that of frontier sampling. It is easy to show that both have computation complexity of $\mathcal{O}(\size{\V[s]}+\size{\E[s]})$ (we can use alias method \citep{alias} for ``Edge'' sampling to achieve such complexity). 

\begin{algorithm}
\caption{Other graph sampling algorithms (``Edge'' and ``RW'')}
\label{algo: other sampler}
\begin{algorithmic}[1]
\renewcommand{\algorithmicrequire}{\textbf{Input:}}
\renewcommand{\algorithmicensure}{\textbf{Output:}}
\Require Training graph $\G\paren{\V,\E}$; Sampling parameters: edge budget $b$; number of roots $r$; random walk length $h$
\Ensure Induced subgraph $\G[s]\paren{\V[s],\E[s]}$

\Function{Edge}{$\G$, $m$}{\color{blue}\Comment{Random edge sampler}\color{black}}
\State $P\paren{\paren{u,v}}\coloneqq \paren{\frac{1}{\text{deg}\paren{u}}+\frac{1}{\text{deg}\paren{v}}}/\sum_{\paren{u',v'}\in\E}\paren{\frac{1}{\text{deg}\paren{u'}}+\frac{1}{\text{deg}\paren{v'}}}$
\State $\E[s]\gets $ $m$ edges randomly sampled from $\E$ according to distribution $P$
\State $\V[s]\gets$ Set of nodes that are end-points of edges in $\E[s]$
\State $\G[s]\gets$ Node induced subgraph of $\G$ from $\V[s]$
\EndFunction

\Function{RW}{$\G$, $r$, $h$}{\color{blue}\Comment{Unbiased random walk sampler}\color{black}}
\State $\mathcal{V}_\text{root}\gets$ $r$ root nodes sampled uniformly at random from $\V$
\State $\V[s]\gets \mathcal{V}_\text{root}$
\For{$v\in \mathcal{V}_\text{root}$}
    \State $u\gets v$
    \For{$d=1$ to $h$}
        \State $u\gets $ Node sampled uniformly at random from $u$'s neighbor
        \State $\V[s]\gets \V[s]\cup\set{u}$
    \EndFor
\EndFor
\State $\G[s]\gets$ Node induced subgraph of $\G$ from $\V[s]$
\EndFunction

\end{algorithmic} 
\end{algorithm}

For the ``Edge'' and ``RW'' samplers, we thus only apply inter-sampler parallelism to achieve scalability. We can use exactly the same inter-sampler parallelization strategy discussed above. The only difference is that each subgraph in the pool $\Set{\G[i]}$ is now obtained by a serial ``Edge'' or ``RW'' sampler.

To further improve the training accuracy with ``Edge'' and ``RW'' samplers, we further integrate the \emph{aggregator normalization} and \emph{loss normalization} techniques \citep{graphsaint} into our implementation. 
Such normalization requires two minor modifications to our training algorithm: 
\begin{itemize}
    \item Pre-processing: Before training, we would need to independently sample a given number of subgraphs to estimate the probability of each $v\in\V$ and $e\in\E$ being picked by the sampling algorithm. The pre-processing can be parallelized by the strategies discussed above. 
    \item Applying the normalization coefficients: with aggregator normalization, the feature aggregation (i.e., $\An_s\cdot \bm{X}_s$) would be based on a re-normalized adjacency matrix. With loss normalization, the loss $\mathcal{L}_s$ would be computed with \emph{weighted} sum for the minibatch nodes. Therefore, the two normalization steps do not make any change on the computation pattern. 
\end{itemize}
}
\section{Parallel Training Algorithm}
\label{sec: para spmm}

We next present parallelization techniques for the forward and backward propagation. Specifically, 
the subgraph based training enables a simple partitioning scheme that ensures near-optimal feature propagation performance. 

\subsection{Computation Kernels in Training}
\label{sec: para feat prop kernel}

After obtaining the subgraphs as minibatches, the GNN computation mainly involves forward and backward propagation along the layers. 
{We first analyze in detail the backward propagation computation for the GraphSAGE model \citep{graphsage}. Then we show that all the four GNN variants presented in Section \ref{sec: gcn background} share the same set of key computation operations. And thus the parallelization strategy can be generally applied to all the models. }
As for the forward propagation, Equations \ref{eq: graphsage forward}, \ref{eq: mlp forward} and \ref{eq: loss forward} have already defined all the operations required for the various layers. Next, we derive the equations for calculating gradients. 

Starting from the minibatch loss $\mathcal{L}_s$, we first compute the gradient with respect to the classifier output on the subgraph nodes $\paren{\X_\text{MLP}}_s$. Then, using chain-rule, we compute the gradients with respect to the variables of the MLP layer and the graph convolution layers (from layer $L$ back to layer $1$). 

For the layer with cross-entropy loss, the gradients are computed by:

\begin{align}
    \nabla_{\paren{\X_\text{MLP}}_s}\mathcal{L}_s = \frac{1}{\size{\V[s]}}\cdot \paren{\bm{Y}_s-\overline{\bm{Y}}_s}
\end{align}

For the MLP layer, the gradients are computed by:

\begin{align}
    \nabla_{\W[MLP]}\mathcal{L}_s=&\paren{\X[L]_s}^\trans\cdot \func[mask]{\nabla_{\paren{\X_\text{MLP}}_s}\mathcal{L}_s}\\
    \nabla_{\X[L]_s}\mathcal{L}_s=&\func[mask]{\nabla_{\paren{\X_\text{MLP}}_s}\mathcal{L}_s\cdot \paren{\W[MLP]}^\trans}
\end{align}

For each graph convolution layer $\ell$, the gradients are computed by:

\begin{align}\label{eq: gcn backward}
    \nabla_{\Ws[\ell]}\mathcal{L}_s=& \paren{\X[\ell-1]_s}^\trans\cdot \func[mask]{\left[\nabla_{\X[\ell]_s}\mathcal{L}_s\right]_{:,0:\frac{1}{2}\f[\ell]}}\\
    %%%%%
    \nabla_{\Wn[\ell]}\mathcal{L}_s=& \paren{\An_s\X[\ell-1]_s}^\trans\cdot \func[mask]{\left[\nabla_{\X[\ell]_s}\mathcal{L}_s\right]_{:,\frac{1}{2}\f[\ell]:\f[\ell]}}\\
    %%%%%%
    \nabla_{\X[\ell-1]_s}\mathcal{L}_s=& \func[mask]{\left[\nabla_{\X[\ell]_s}\mathcal{L}_s\right]_{:,0:\frac{1}{2}\f[\ell]}}\cdot \paren{\Ws[\ell]}^\trans\\
    +& \paren{\An_s}^\trans\cdot\func[mask]{\left[\nabla_{\X[\ell]_s}\mathcal{L}_s\right]_{:,\frac{1}{2}\f[\ell]:\f[\ell]}}\cdot \paren{\Wn[\ell]}^\trans
\end{align}

From the equations of forward and backward propagation, we observe that the GraphSAGE computation consists of three kernels: 
\begin{itemize}
    \item Feature / gradient propagation in the sparse subgraph -- e.g., $\An_s\X[\ell]_s$;
    \item Dense weight transformation on the feature / gradient -- e.g., $\X[\ell-1]_s\Ws[\ell]$;
    \item Sparse adjacency matrix transpose -- i.e., $\paren{\An_s}^\trans$.
\end{itemize}

{
In fact, the above three are also the key operations for GCN \citep{gcn}, MixHop \citep{mixhop} and GAT \citep{gat}. 
For GCN \citep{gcn}, the forward propagation only contains one pass as compared to the two paths in GraphSAGE (i.e., the two paths being concatenated by the ``$\|$'' operation). Therefore, in the backward propagation, we replace $\An[s]$ with $\hat{\bm{A}}_s$ and only keep the terms containing $\hat{\bm{A}}_s$ in Equation \ref{eq: gcn backward}. For example, we have $\nabla_{\X_s^{(\ell-1)}}\mathcal{L}_s=\paren{\hat{\bm{A}}}^\trans\cdot \func[mask]{\nabla_{\X_s^{(\ell)}}\mathcal{L}_s}\cdot \paren{\W^{(\ell)}}^\trans$.

For MixHop \citep{mixhop}, each layer in the forward propagation consists of $K$ paths as compared to the two paths in GraphSAGE. Therefore, we need to introduce the $\paren{\hat{\bm{A}}}^k$ terms (where $1\leq k \leq K$) to Equation \ref{eq: gcn backward} in the backward pass. For example, we need $\paren{\hat{\bm{A}}_s}^2\X[\ell-1]_s$ to compute $\nabla_{\W_2^{(\ell)}}\mathcal{L}_s$. 
Further note that $\paren{\hat{\bm{A}}_s}^2\X[\ell-1]_s= \hat{\bm{A}}_s\cdot \paren{\hat{\bm{A}}_s\X[\ell-1]_s}$. And even though $\bm{A}_s$ is sparse, the product $\hat{\bm{A}}_s\X[\ell-1]_s$ is again a dense matrix. So the forward and backward propagation for MixHop does not involve sparse-sparse matrix multiplication and the MixHop computation can still be covered by the three key operations listed above. 

For GAT \citep{gat}, in the forward pass, we need to compute the attention values for each element in the subgraph adjacency matrix. Such computation according to Equation \ref{eq: gat attention} only involves dense algebra. After obtaining the attention adjacency matrix, the rest of the propagation by Equation \ref{eq: gat forward} is the same as that of GCN. 
In the backward pass, according to chain rule, we can still break down the computation steps following the logic in the forward pass. For example, to obtain the gradient with respect to attention parameters $\bm{a}$, we first obtain the gradients with respect to the attention matrix $\bm{A}_{\text{att}}^{(\ell-1)}$ by a series of dense matrix operations on $\X[\ell-1]$, $\nabla_{\X[\ell]}\mathcal{L}_s$ and $\W$. Then we obtain the gradient with respect to $\bm{a}$ based on the gradient with respect to $\bm{A}_{\text{att}}^{(\ell-1)}$. Even though the mathematical expression for the GAT gradient computation is more complicated, it is easy to see that all the operations involved are again covered by the three key operations listed above. 

In summary, if we can efficiently parallelize the three operations listed above, we are automatically able to execute the full forward and backward propagation for the four GNNs. 
}
We present our method for transposing the sparse adjacency matrix in Section \ref{sec: adj trans} and the techniques for parallel feature propagation in Section \ref{sec: para feat prop}. 
Now consider the dense matrix multiplication involved in the weight transformation step. Since this operation is a standard BLAS level 2 routine, it can be efficiently parallelized using standard libraries such as Intel\textsuperscript{\small\textregistered}
 MKL \citep{MKL}. 

{In the following, we use $\An[s]$ to represent the subgraph adjacency matrix used in each GNN layer. For different models, the $\An[s]$ may be replaced by $\hat{\bm{A}}_s$ or $\bm{A}_{\text{att}}$. }

\subsection{Transpose of the Sparse Adjacency Matrix}
\label{sec: adj trans}

Since we assume the training graph and the sampled subgraphs are undirected, the transpose of the subgraph adjacency matrix $\paren{\An_s}^\trans$ can be performed efficiently with low computation and space complexity. We first discuss the serial implementation before moving forward to the parallel version. 

Suppose the original adjacency matrix $\An$ is represented in the CSR format, consisting of a size-$\size{\V[s]+1}$ index pointer array (\textsc{Indptr}), a size-$\size{\E[s]}$ indices array (\textsc{Indices}) and a size-$\size{\E[s]}$ data array (\textsc{Data}). 
For an undirected graph, if edge $\paren{u,v}\in \E[s]$, then $\paren{v,u}\in \E[s]$. 
Therefore, the index pointer and the indices arrays of $\An[s]$ are identical as the ones of $\paren{\An[s]}^\trans$. 
To transpose $\An[s]$ thus means to generate a new data array by permuting the original \textsc{Data} of the CSR of $\An_s$ . 

\begin{algorithm}
\caption{Transpose of the subgraph adjacency matrix}
\label{algo: adj trans}
\begin{algorithmic}[1]
\renewcommand{\algorithmicrequire}{\textbf{Input:}}
\renewcommand{\algorithmicensure}{\textbf{Output:}}
\Require Original adjacency matrix $\An[s]$ represented by the CSR format
\Ensure Transposed adjacency matrix $\paren{\An[s]}^\trans$ represented by the CSR format
\State $\textsc{Indptr}, \textsc{Indices}, \textsc{Data} \gets $ CSR arrays of $\An[s]$
\State $\textsc{DataTrans}\gets$ array of size $\size{\E[s]}$ initialized to $\func[INV]{}$
\State $\textsc{PtrData}\gets$ array of size $\size{\V[s]}$ initialized to $\textsc{Indptr}[:\size{\V[s]}]$
\For {$v$ from $0$ to $\size{\V[s]}-1$}
    \For {$j$ from $\textsc{Indptr}[v]$ to $\textsc{Indptr}[v+1]$}
        \State $u\gets \textsc{Indices}[j];\qquad a\gets \textsc{Data}[j];$
        \State $\textsc{DataTrans}[\textsc{PtrData}[u]]\gets a$
        \State Increment $\textsc{PtrData}[u]$ by $1$\color{blue}\Comment{Record the next position to append}\color{black}
    \EndFor
\EndFor
\State \Return Transposed matrix $\paren{\An[s]}^\trans$ from \textsc{Indptr}, \textsc{Indices}, $\textsc{DataTrans}$
\end{algorithmic} 
\end{algorithm}

We propose to generate the permuted data array for $\paren{\An_s}^\trans$ by a single pass of \textsc{Indptr} and \textsc{Indices} of $\An_s$. 
Our algorithm relies on a weak assumption on \textsc{Indices} of $\An[s]$: for any node $v$, we assume its neighbor IDs in the indices array, $\textsc{Indices}\left[\textsc{Indptr}[v]:\textsc{Indptr}[v+1]\right]$, is sorted in ascending order. 
The transpose operation is shown in Algorithm \ref{algo: adj trans}. 
The correctness of the algorithm can be reasoned as follows. 
Suppose a column $v$ of the original adjacency matrix has $n$ non-zeros denoted as $\left[\An[s]\right]_{u_i,v}=a_i$, where $1\leq i \leq n$ and the node IDs satisfy $u_i < u_j$ for $i<j$.
When we traverse the CSR of $\An[s]$ (lines 4 to 5), we will read $a_i$ before $a_j$ if the node IDs have $u_i < u_j$. 
After transpose, the neighbor data -- $a_1,\ldots, a_n$ -- should be placed in a continuous subarray $\textsc{Data}\left[\textsc{Indptr}[v]:\textsc{Indptr}[v+1]\right]$ of $\paren{\An[s]}^\trans$. 
In addition, $a_i$ should locate to the left of $a_j$ if $u_i < u_j$. 
Therefore, once reading $a_i$ of the edge $\paren{u_i,v}$ from $\An[s]$, we can simply append $a_i$ to $v$'s data subarray of the transposed CSR.

The computation and space complexity of Algorithm \ref{algo: adj trans} are $\mathcal{O}\paren{\size{\V[s]}+\size{\E[s]}}$ and $\mathcal{O}\paren{\size{\E[s]}}$ respectively, which are low compared with other operations in training. 
We parallelize the adjacency matrix transpose at the subgraph level. During sampling, each of the $p_\text{inter}$ processors sample one subgraph and permute the corresponding \textsc{Data} array by Algorithm \ref{algo: adj trans}. 
The information of the original and transposed subgraphs are all stored in the pool of $\Set{\G[i]}$ (Section \ref{sec: para sample inter}), to be later consumed by the GNN layer propagation. 

\subsection{Parallel Feature Propagation within Subgraph}
\label{sec: para feat prop}

During training, each node in the graph convolution layer $\ell$ pulls features from its neighbors, along the layer edges. Essentially, the operation of $\An\X[\ell-1]_s$ can be viewed as feature propagation within the subgraph $\G[s]$. 

A similar problem, label propagation within graphs, has been extensively studied in the literature. State-of-the-art methods based on
vertex-centric \citep{prop_blocking}, edge-centric \citep{xstream} and partition-centric \citep{gpop} paradigms perform node partitioning on graphs so that processors can work independently in parallel. The work in \citep{hidden_dim} also performs label partitioning along with graph partitioning when the label size is large. 
In our case, we borrow the above ideas to allow two dimensional partitioning along the graph as well as the feature dimensions. 
However, we also realize that the aforementioned techniques may lead to sub-optimal performance in our {GNN} based feature propagation, due to two reasons:

\begin{itemize}
    \item The propagated data from each node is a long feature vector (consisting of hundreds of elements) rather than a small scalar label. 
    \item Our graph sizes are small after graph sampling, so partitioning of the graph may not lead to significant advantage. 
\end{itemize}

In the following, we analyze the computation and communication costs of feature propagation after graph and feature partitioning. We temporarily ignore load-imbalance and partitioning overhead, and address them later on. 
%For conciseness, we use $\G(\V,\E)$ to refer to $\Gt[sub]\paren{\Vt[sub],\Et[sub]}$.

Suppose we partition the subgraph into $Q_v$ number of disjoint node partitions $\Set{\V^{i}_s|0\leq i\leq Q_v-1}$. Let the set of nodes that send features to $\V^{i}_s$ be $\V_{\text{src}}^{i}=\Set{u|(u,v)\in\E[s]\land v\in \V^{i}_s}$. Note that $\V^{i}_s\subseteq\V_{\text{src}}^{i}$, since we follow the design in \citep{graphsage} to add a self-connection to each node. 
We further partition the feature vector $\bm{x}_v\in \mathbb{R}^f$ of each node $v$ into $Q_f$ equal parts $\Set{\bm{x}_v^{i}|0\leq i\leq Q_f-1}$. 
Each of the processors is responsible for propagation of $\X^{i,j}_s=\Set{\bm{x}_v^{j}|v\in \V_{\text{src}}^{i}}$, flowing from $\V_{\text{src}}^{i}$ into $\V^{i}$ (where $0\leq i\leq Q_v-1$ and $0\leq j\leq Q_f-1$).

Define $\gamma_{v}=\frac{\size{\V^{i}_\text{src}}}{\size{\V}}$ as a metric reflecting the graph partitioning quality. While $\gamma_{v}$ depends on the partitioning algorithm, it is always bound by $\frac{1}{Q_v}\leq \gamma_{v}\leq 1$.

Let $n=\size{\V[s]}$ and $f=\size{\bm{x}_v}$. So $\size{\V^{i}}=\frac{n}{Q_v}$ and $\size{\bm{x}_v^{i}}=\frac{f}{Q_f}$. 
%and the feature vector $\bm{h}_v$ of each vertex $v$ into $q$ partitions, $\Set{\bm{h}_v^{(i)}|0\leq i\leq q-1}$. Each processor computes the propagation of 

In our performance model, we assume $p$ processors operating in parallel. Each processor is associated with a private fast memory (i.e., cache). The $p$ processors share a slow memory (i.e., DRAM). Our objective in partitioning is to minimize the overall processing time in the parallel system. After partitioning, each processor owns $\frac{Q_v\cdot Q_f}{p}$ number of $\X^{i,j}_s$, and propagates its $\X^{i,j}_s$ into $\V^{i}$. Due to the irregularity of graph edge connections, accesses into $\X^{i,j}_s$ are random. On the other hand, using the CSR format, the neighbor lists of nodes in $\V^{i}$ can be streamed into the processor, without the need to stay in cache. In summary, an optimal partitioning scheme should:
\begin{itemize}
    \item Let each $\X^{i,j}_s$ fit into the fast memory;
    \item Utilize all of the available parallelism in the system;
    \item Minimize the total computation workload;
    \item Minimize the total slow-to-fast memory traffic;
    \item Balance the computation and communication load among the processors.
\end{itemize}

Each round of feature propagation has $\frac{n}{Q_v}\cdot d\cdot \frac{f}{Q_f}$ computation, and $2\cdot \frac{n}{Q_v}\cdot d+8\cdot n\cdot \gamma_{v}\cdot \frac{f}{Q_f}$ communication (in bytes)\footnote{Given that sampled graphs are small, we use \texttt{INT16} to represent the node indices. We use \texttt{DOUBLE} to represent each feature value.}. Computation and computation over $Q_v\cdot Q_f$ rounds are:

\begin{align}
    g_\text{comp}(Q_v,Q_f)&=n\cdot d\cdot f\\
    g_\text{comm}(Q_v,Q_f)&=2\cdot Q_f\cdot n\cdot d+8\cdot Q_v\cdot n\cdot f\cdot \gamma_v
\end{align}

Note that $g_\text{comp}(Q_v,Q_f)$ is not affected by the partitioning scheme. We thus formulate the following \textit{communication minimization problem}:

\begin{align}\label{eqn:opt_part}
%\begin{aligned}
& \underset{Q_v,Q_f}{\text{minimize}}
& & g_\text{comm}(Q_v,Q_f)=2Q_f\cdot nd + 8Q_v\cdot nf\gamma_v \\
& \text{subject to}
& & Q_v Q_f\geq p;\quad \frac{8nf\gamma_v}{Q_f}\leq S_\text{cache};\quad Q_v,Q_f\in \mathbb{Z}^+;
%\end{aligned}
\end{align}
Next, we prove that \emph{without any graph partitioning} we can obtain a 2-approximation for this optimization problem for small subgraphs.
\begin{thm}
\label{thm: label prop}
$Q_v=1, Q_f=\max\left\{p, \frac{8nf}{S_\text{cache}}\right\}$ results in a 2-approximation of the communication minimization problem (Equation \ref{eqn:opt_part}), for $p\leq \frac{4f}{d}$ and $2nd \leq S_\text{cache}$, irrespective of the partitioning algorithm.
\end{thm}

\begin{pf}

% Assume that $\gamma^*$ is the value of $\gamma_P$ for the optimal value of $P$. We define $g'(P, Q) = 2Qnd + 8Pnf\gamma^*$. Clearly, $\min g'(P, Q) \leq \min g_\text{comm}(P, Q)$, with the additional constraint $P \geq1/\gamma^*$, and relaxing $P, Q \in \mathbb{Z}^+$ to $P, Q \in \mathbb{R}$, and $P, Q \geq 1$.

% Based on the constraints it is clear that the optimal will be at either of the following points:
% \paragraph{Case 1 $(1/\gamma^*, \frac{8nf\gamma^*}{S_\text{cache}})$} This is a feasible point if $PQ \geq C$, i.e., $\frac{1}{\gamma^*} \frac{8nf\gamma^*}{S_\text{cache}} \geq C$, i.e., $\frac{8nf}{S_\text{cache}}\geq C$.
% Therefore, our solution is 
% \begin{align*}
%     &g_\text{comm}\left(1, \frac{8nf}{S_\text{cache}}\right) = 2nd\frac{8nf}{S_\text{cache}} + 8nf \\
%     &= 8nf\left(\frac{2nd}{S_\text{cache}} + 1\right) \leq 8nf(1+1) = 16nf.
% \end{align*}
% And,
% \begin{align*}
% g'\left(1/\gamma^*, \frac{8nf\gamma^*}{S_\text{cache}}\right) &= 2nd\frac{8nf\gamma^*}{S_\text{cache}} + 8nf \geq 8nf.
% %&\geq 2ndC\gamma^* + 8nf \geq 2nd + 8nf.
% \end{align*}

% \paragraph{Case 2 $(\frac{1}{\gamma^*}, C\gamma^*)$} This is a feasible point, if $C\gamma^* \geq \frac{8nf\gamma^*}{S_\text{cache}} \implies C \geq \frac{8nf}{S_\text{cache}}$.
% Therefore, our solution is 
% \begin{align*}
% g_\text{comm}\left(1, C\right) =  2ndC + 8nf \leq 16nf.
% \end{align*}
% This is due to $C \leq 4f/d$. And,
% \begin{align*}
%     g'\left(\frac{1}{\gamma^*}, C\gamma^*\right) &=
%     2ndC\gamma^* + 8nf \geq 8nf.
% \end{align*}

Note that since $Q_v, Q_f \geq  1$ and $\gamma_v \geq 1/Q_v$, $\forall Q_v, Q_f$:
\begin{align*}
    g_\text{comm}(Q_v, Q_f) \geq 2Q_f nd + 8Q_v nf\frac{1}{Q_v} \geq 8nf.
\end{align*}
Set $Q_v=1$ and $Q_f = \max\left\{p, \frac{8nf}{S_\text{cache}}\right\}$. Clearly, $\gamma_v = 1$.
\paragraph{Case 1, $p \geq \frac{8nf}{S_\text{cache}}$} In this case, $Q_f=p \geq 8nf/S_\text{cache}$. Thus both constraints are satisfied. And,
 \begin{align*}
 &g_\text{comm}\left(1, p\right) =  2ndp + 8nf \\&= 8nf\left(\frac{pd}{4f} + 1 \right)
 \leq  8nf\cdot (1+1) = 16nf
 \end{align*}
due to $p \leq 4f/d$.

\paragraph{Case 2, $p \leq \frac{8nf}{S_\text{cache}}$} In this case, $Q_f = 8nf/S_\text{cache}$ is a feasible solution. And,
\begin{align*}
     &g_\text{comm}\left(1, \frac{8nf}{S_\text{cache}}\right) = 2nd\frac{8nf}{S_\text{cache}} + 8nf \\
    &= 8nf\left(\frac{2nd}{S_\text{cache}} + 1\right) \leq 8nf\cdot (1+1) = 16nf
\end{align*}
due to $2nd\leq S_\text{cache}$.

In both cases, the approximation ratio of our solution is ensured to be: 

$$
\frac{g_\text{comm}\left(1, \max\left\{p, \frac{8nf}{S_\text{cache}}\right\}\right)}{\min g_\text{comm}(Q_v, Q_f)} \leq \frac{16nf}{8nf} = 2$$

Note that this holds for $S_\text{cache}\geq 2nd$. So for a cache size of 256KB, number of edges in the subgraph (i.e., $nd$) can be up to 128K. 
Such upper bound on $\size{\E[s]}$ can be met by the subgraphs in consideration. Also, since $f \gg d$, the condition $p \leq 4f/d$ holds for most of the shared memory platforms in the market. 
{Note that the above theorem is derived by a simple lower bounding on the ratio $\gamma_v$ for any (including the optimal) partitioning scheme. However, finding such optimal partitioning is computationally infeasible even on small subgraphs, since there are exponential number of possible partitioning. We thus do not provide experimental evaluation on this theorem. }

\end{pf}
%We do not directly solve the above optimization problem. Instead, we propose a simple partitioning scheme which always sets $P=1$, and provide a bound of $g_\text{comm}(1,q)_\text{opt}$ as compared to $g_\text{comm}(p,q)_\text{opt}$. Define the ratio $\alpha=\frac{g_\text{comm}(1,q)_\text{opt}}{g_\text{comm}(p,q)_\text{opt}}$

%Consider two corner cases, when $\gamma_p=1$ and $\gamma_p=\frac{1}{p}$.

%\paragraph*{Case 1: $\gamma_p=1$}
% By constraint 1, $g_\text{comm}(p,q)=2qnd+8pnf\geq 2\frac{C}{p}nd+8pnf\geq 4n\sqrt{Cdf}$. By constraint 2, $g_\text{comm}(p,q)\geq 2\frac{8nf}{S_\text{cache}}\cdot nd+8nf$. 
% On the other hand, for our feature only partitioning scheme, $g_\text{comm}(1,q)_\text{opt}=2\cdot  \max\{C,\frac{8\cdot n\cdot f}{S_\text{cache}}\}\cdot n\cdot d+8\cdot n\cdot f$. When $C<\frac{8\cdot n\cdot f}{S_\text{cache}}$, $\alpha=1$. When $C\geq \frac{8 n f}{S_\text{cache}}$, $\alpha\leq \frac{C d+4 f}{\max\{2\sqrt{Cdf},8nfd/S_\text{cache}+4f\}}$. 

% \paragraph*{Case 2: $\gamma_p=\frac{1}{p}$}

% In this case, $g_\text{comm}(p,q)\geq 2nd+8nf$. And for the feature only partitioning scheme, $g_\text{comm}(1,q)_\text{opt}=2Cnd+8nf$. We get $\alpha=\frac{Cd+4f}{d+4f}$. 

Using typical  values $n\leq 8000$, $f=512$, and $d=15$, then for up to {$p\leq \frac{4f}{d}=136$ cores}\footnote{Note that $d$ here refers to the average degree of the sampled graph rather than the training graph. Thus, $d$ value here is set to be lower than that in Section \ref{sec: para sample}.}, the total slow-to-fast memory traffic under feature only partitioning is less than 2 times the optimal. 
Recall the two properties (see the beginning of this section) that differentiate our case with the traditional label propagation. Because the graph size $n$ is small enough, we can find a feasible $Q_f\in\mathbb{Z}^+$ solution to satisfy the cache constraint $\frac{8nf}{Q_f}\leq S_\text{cache}$. Because the value $f$ is large enough, we can find enough number of feature partitions such that $Q_f\geq p$. 
%Figure \ref{fig: partition alpha} shows the change in $\alpha$ with respect to number of processor cores $C$, for $f\in\left[ 256,1024 \right]$ and $d=15$. 
%For state-of-the-art shared memory system, we can guarantee that the total slow-to-fast memory traffic is no more than 2 times the optimal one, for all possible partitioning schemes and types of subgraphs. 
Algorithm \ref{algo: gsaint training feat prop} specifies our feature propagation.

\begin{algorithm}
\caption{Feature propagation within sampled graph}
\label{algo: gsaint training feat prop}
\begin{algorithmic}[1]
\renewcommand{\algorithmicrequire}{\textbf{Input:}}
\renewcommand{\algorithmicensure}{\textbf{Output:}}
\Require Subgraph $\G[s]\paren{\V[s],\E[s]}$ with adjacency matrix $\An[s]$; Node feature matrix $\X[\ell-1]_s$; Cache size $S_\text{cache}$; Number of processors $p$
\Ensure Feature matrix $\X[\ell]_s$
\State $n\gets \size{\V[s]};\qquad f\gets \text{length of the feature vector of a node};$
\State $Q_f\gets \max\left\{p,\frac{8nf}{S_\text{cache}}\right\};\qquad f'\gets f/Q_f;$
\State Column-partition $\X[\ell-1]_s$ into $Q_f$ equal-size parts $\left[\X_{s}^{\paren{\ell-1}}\right]_{:,\,i\cdot f':(i+1)\cdot f'}$
\For{$r=0$ to $Q_f/p-1$}
    \ParFor{$j=0$ to $p-1$}%\color{blue}\Comment{}\color{black}
        %\State Propagation of $\Set{\bm{h}_v^{(i+r\cdot p)}|v\in \V}$ into $\V$
        \State $i\gets r+j\cdot Q_f/p$
        \State $\left[\X[\ell]_s\right]_{:,\,i\cdot f':(i+1)\cdot f'}\gets \An[s]\cdot \left[\X_s^{\paren{\ell-1}}\right]_{:,\,i\cdot f':(i+1)\cdot f'}$
    \EndParFor
\EndFor
\State \Return $\X[\ell]_s$
\end{algorithmic} 
\end{algorithm}

Lastly, the feature only partitioning leads to two more important benefits. Since the graph is not partitioned, load-balancing (with respect to both computation and communication) is optimal across processors. Also, our partitioning incurs almost zero pre-processing overhead since we only need to extract continuous columns to form sub-matrices. In summary, the feature propagation in our graph sampling-based training achieves 
\begin{enumerate*}
    \item Minimal computation;
    \item Optimal load-balancing;
    \item Zero pre-processing cost;
    \item Low communication volume.
\end{enumerate*}

\section{Runtime Scheduling}
\label{sec: scheduler}

\subsection{Computation Order of Layer Operations}
\label{sec: scheduler ordering}
Both the forward and backward propagation of {GNN} layers (Equations \ref{eq: gcn forward}, \ref{eq: graphsage forward}, \ref{eq: mixhop forward}, \ref{eq: gat forward} and \ref{eq: gcn backward}) involve multiplying a chain of three matrices. 
{Given a chain of matrix multiplication, it is known that different orders of computing the chain leads to different computation complexity. In general, we can use dynamic programming techniques to obtain the optimal order corresponding to the lowest computation complexity \citep{mcm}. Specifically, for our training problem, we have a chain of three matrices whose sizes and densities are known once the subgraphs are sampled. }
Consider a sparse matrix $\bm{A}\in\mathbb{R}^{n\times n}$ (with density $\delta$), and two dense matrices $\bm{W}_1\in \mathbb{R}^{n\times f_1}$ and $\bm{W}_2\in\mathbb{R}^{f_1\times f_2}$. 
To calculate $\bm{A}\bm{W}_1\bm{W}_2$, {there are two possible computation orders. Order 1 of $\paren{\bm{A}\bm{W}_1}\bm{W}_2$ computes the partial result $\bm{P}=\bm{A}\W_1$ first and then computes $\bm{P}\W[2]$}. This order of computation requires $\delta\cdot n^2 \cdot f_1+n\cdot f_1\cdot f_2$ {Multiply-ACcumulate (MAC)} operations. {Order 2 of $\bm{A} \paren{\bm{W}_1 \bm{W}_2}$ computes the partial result $\bm{P}=\W[1]\W[2]$ first and then computes $\bm{A}\bm{P}$.} This order requires $\delta \cdot n^2 \cdot f_2 + n\cdot f_1\cdot f_2$ MAC operations. Therefore, if $f_1<f_2$, order 1 is better. Otherwise, we should use order 2.  
Similarly, suppose $\bm{W}_3\in\mathbb{R}^{n\times f_3}$ and our target is $\paren{\bm{W}_1}^\trans\bm{A}\bm{W}_3$. Then order 1 of $\paren{\bm{A}\bm{W}_1}^\trans\bm{W}_3$ is better than order 2 of 
$\paren{\bm{W}_1}^\trans\paren{\bm{A}\bm{W}_3}$ if and only if $f_1< f_3$. 

Consider a GraphSAGE layer $\ell$. If $\f[\ell-1]<\f[\ell]$, we should use order 1 to calculate the forward propagation of Equation \ref{eq: gcn forward}, order 1 to calculate $\nabla_{\Ws[\ell]}\mathcal{L}_s$  of Equation \ref{eq: gcn backward} and order 2 to calculate $\nabla_{\X[\ell-1]_s}\mathcal{L}_s$ of Equation \ref{eq: gcn backward}. 

Note that the decision of the scheduler only relies on the dimension of the matrices, and thus can be made during runtime at almost no cost. In addition, the partitioning strategy presented in Section \ref{sec: para feat prop} does not rely on any specific computation order. In summary, the light-weight scheduling algorithm reduces computation complexity without affecting scalability. 

\subsection{Scheduling the Feature Partitions}
\label{sec: schedule clipping}

After partitioning the feature matrix (Section \ref{sec: para feat prop}), the question still remains how to schedule these partitions for further performance optimization. 
Ideally, since the operations on the partitions are completely independent, any scheduling would lead to identical performance. 
However, in reality, the partitions may undesirably interact with each other due to ``false sharing'' of data in private caches. 
If the size of each feature partition is not divisible by the cacheline size, then in the private cache of the processor owning partition $i$, there may be one cachline containing data of both partitions $i$ and $i+1$, and another cacheline containing data of both partitions $i-1$ and $i$. 
Therefore, if the partitions $i-1$, $i$ and $i+1$ are computed concurrently, there may be undesirable data eviction to keep the three caches clean. 
%To completely decouple the processing of different partitions and avoid ``false sharing'', 
So the scheduler should try not to dispatch adjacent partitions at the same time, and we follow the processing order as specified by lines 5 and 6 of Algorithm \ref{algo: gsaint training feat prop} to achieve this goal.

When the number of processors is large or the number of feature partitions is small (i.e., line 4 of Algorithm \ref{algo: gsaint training feat prop} finishes in one iteration), it is inevitable to process adjacent partitions in parallel. 
On the other hand, note that if the partition size is divisible by the cacheline size, we can avoid ``false sharing'' regardless of the scheduling. The partition size equals $w\cdot \size{\V[s]}\cdot f/Q_f$, where $w$ specifies the word-length. 
Suppose the cacheline size is $S_\text{cline}$. Then our goal is to make $\size{\V[s]}$ divisible by $S_\text{cline}/w$. 
For example, if we use double-precision floating point numbers in training and the cacheline size is 128 bytes, then we can clip the number of subgraph nodes to be divisible by $16$. Considering that $\size{\V[s]}$ is in the order of $10^3$, such clipping has negligible effect on the subgraph connectivity and the training accuracy. 
The node clipping can be performed before the induction step (line 9 of Algorithm \ref{algo: frontier sampling}) by randomly dropping nodes in $\V[s]$. Therefore, the clipping step incurs almost zero cost.

\subsection{Overall Scheduler}

\begin{algorithm}
\caption{Runtime scheduling (with Frontier  sampling as an example)}
\label{algo: gsaint scheduler}
\begin{algorithmic}[1]
\renewcommand{\algorithmicrequire}{\textbf{Input:}}
\renewcommand{\algorithmicensure}{\textbf{Output:}}
\Require Training graph $\G\paren{\V,\E,\X}$; Ground truth labels $\overline{\bm{Y}}$; $L$-layer GNN model; Sampler parameters $m,n,\eta$; Parallelization parameters $p_{\text{inter}}, p_{\text{intra}}$
\Ensure Trained GNN weights
%\State Construct $L$-layer GCN to get $\Set{\V^{(\ell)}}$, $\Set{\E^{(\ell)}}$
\State $\Set{\G_i}\gets \emptyset$\color{blue}\Comment{Set of unused subgraphs}\color{black}
\While{not terminate}\color{blue}\Comment{Iterate over minibatches}\color{black}
    \If{$\Set{\G_i}$ is empty}
        \ParFor{$p=0$ to $p_{\text{inter}}-1$}
            \State  $\G[s]\gets\func[SAMPLE]{\G\paren{\V,\E}}$ with $p_\text{intra}$; Clip nodes by cacheline size
            \State Transpose $\G[s]$ by permuting the \textsc{Data} array of the CSR
            \State Add $\G[s]$ and its transposed array \textsc{Data} to the pool $\Set{\G[i]}$
            %\State $\Set{\G_i}\gets \Set{\G_i}\cup \text{SAMPLE}_{\text{G}}(\G,m,n,\eta,p_{\text{intra}})$
        \EndParFor
    \EndIf
    \State $\G[s]\gets $ Subgraph popped out from $\Set{\G_i}$
    \State Construct GNN on $\G[s]$
    \State {Determine the order of matrix chain multiplication by Section \ref{sec: scheduler ordering}}
    \State Parallel forward and backward propagation on GNN
\EndWhile
%\STATE $\Vls^{(L)}\gets $ Sampled nodes in $\V^{(L)}$ (uniformly at random)\\
\State \Return Trained GNN weights
\end{algorithmic} 
\end{algorithm}

Algorithm \ref{algo: gsaint scheduler} presents the overall training scheduler. As discussed in Section \ref{sec: para sample inter}, multiple samplers can be launched in parallel without any data dependency. This is shown by lines 4 to 8. Note that the clipping follows the objective specified in Section \ref{sec: schedule clipping} and the transpose of $\G[s]$ follows Algorithm \ref{algo: adj trans}. 
After the GNN is constructed, the forward and backward propagation operations are parallelized by the techniques presented in Section \ref{sec: para spmm}. 
{The scheduler performs two decisions based on the sampled subgraphs. The first decision (during runtime) is to perform node clipping to improve cache performance (Section \ref{sec: schedule clipping}). The second decision (statically performed before the actual training) is to determine the order of matrix chain multiplication in both forward and backward propagation to reduce computation complexity (Section \ref{sec: scheduler ordering}). }

Note that our scheduler is a general one, in the sense that the training can replace the frontier sampler with any other graph sampling algorithm in a plug-and-play fashion. The processing by the scheduler has negligible overhead. 
%%%%%%%%
% yelp %
%%%%%%%%
% subgraph result: mic=0.5585  mac=0.3268  loss=0.0477
% subgraph config:
%	subgraph size: 12000
%	learning rate: 0.05
%	epoch: 150
% model config:
%	dim1: 2048
%	dim2: 2048

\section{Experiments}
\label{sec: exp}

\subsection{Experimental Setup}
\label{sec: exp setup}
We conduct experiments on 4 large scale real-world graphs as well as on synthetic graphs. Details of the datasets are described as follows:
\begin{itemize}
    \item PPI \citep{ppi_reddit}: A protein-protein interaction graph. A node represents a protein and edges represent protein interactions. 
    \item Reddit \citep{ppi_reddit}:  A post-post graph. A node represents a post. An edge exists between two posts if the same user has commented on both posts.
    \item Yelp \citep{yelp, graphsaint}: A social network graph. A node is a user. An edge represents friendship. Node attributes are user comments converted from text using Word2Vec \citep{word2vec}.
    \item Amazon \citep{graphsaint}: An item-item graph. A node is a product sold by Amazon. An edge is present if two items are bought by the same customer. Node attributes are converted from bag-of-words of text item descriptions using singular value decomposition (SVD).
    \item Synthetic graphs: Graphs generated by Kronecker generator \citep{kronecker}. We follow the setup in \citep{kronecker} and set the initiator matrices to be proportional to the 2 by 2 matrix $[[0.9, 0.5], [0.5, 0.1]]$. We generate two sets of Kronecker graphs. The first set consists of graphs with fixed average degree of 16 and number of nodes equals to $2^{20}$, $2^{21}$, $2^{22}$, $2^{23}$, $2^{24}$ and $2^{25}$. The second set consists of graphs with $2^{20}$ nodes and the average degree equals to $8$, $16$, $32$, $64$, $128$, $256$ and $512$.
    \end{itemize}
The PPI and Reddit datasets are standard benchmarks used in \citep{gcn,graphsage,sgcn,asgcn,fastgcn}. %Note that the Reddit graph is currently the largest one evaluated by state-of-the-art embedding methods, such as \citep{graphsage,fastgcn,kdd_gcn}. 
The larger scale graphs, Yelp and Amazon, are processed and evaluated in \citep{ipdps,graphsaint}. We use the set of four real-world graphs for a thorough evaluation on accuracy, efficiency and scalability. Table \ref{tab:dataset} shows the specification of the graphs. We use ``fixed-partition'' split, and the ``Train/Val/Test'' column shows the percentage of nodes in the training, validation and test sets. ``Classes'' shows the total number of node classes (i.e., number of columns of $\bm{Y}$ and $\overline{\bm{Y}}$ in Equation \ref{eq: loss forward}). 
For synthetic graphs, we can only generate the graph topology. The node attributes and the class memberships are filled by random numbers. 

%PPI is a protein-protein interaction graph provided by GraphSAGE. Reddit is a social network graph from an online forum which is both used in GraphSAGE and FastGCN. These two datasets are the largest two academic datasets used in the state-of-art graph embedding algorithms. 

%As for larger dataset, we preprocess yelp and amazon. Yelp is a user-user social graph generated from the 2018 Yelp data challenge dataset gathered from an online business rating website. Each node in the yelp dataset corresponds to one user while the edges represent two users marked each other as friends. The feature is obtained through [Word2Vec] from the raw comments posted by the user. Each user is classified to multiple types of business that this user has visited. Amazon is an item-item social graph generated from a user-item bipartite graph(provided by HIVE, do i say that here?). The edge in Amazon graph represents that two items have bought by at least a number of common users. The feature is obtained using singular value decomposition (SVD) from the original word-of-bags representation of the comments on the items. And each item is classified to multiple types. 

\begin{table}[!ht]
\caption{Dataset Statistics}
    \centering
    \resizebox{\columnwidth}{!}{
    \begin{tabular}{rccccc}
        \toprule
        Dataset & Nodes & Edges & Attributes & Classes & Train/Val/Test\\
        \midrule
        \midrule
        PPI & 14,755 & 225,270 & 50 & 121 (M) & 0.66/0.12/0.22\\
        Reddit & 232,965 & 11,606,919 & 602 & 41 (S) & 0.66/0.10/0.24\\
        Yelp & 716,847 & 6,977,410 & 300 & 100 (M) & 0.75/0.15/0.10\\
        Amazon & 1,598,960 & 132,169,734 & 200 & 107 (M) & 0.80/0.05/0.15\\
        \midrule
        Synthetic & $2^{20}$ - $2^{25}$ & $2^{23}$ - $2^{30}$ & 50 & 2 (S) & 0.50/0.25/0.25\\
        \bottomrule
    \end{tabular}}
    \\
    \vspace{.2cm}
    {\scriptsize{}{*} The (M) mark stands for \textbf{M}ulti-class classification, while (S) stands for \textbf{S}ingle-class.}{\scriptsize \par}
    \label{tab:dataset}
\end{table}

For our graph sampling based GNN training, we open-source two implementations in \verb|Python| (with Tensorflow) and \verb|C++| (with OpenMP), respectively\footnote{Code available at: \texttt{https://github.com/GraphSAINT/GraphSAINT}}. We use the \verb|Python| (Tensorflow) version for single threaded accuracy evaluation in Section \ref{sec: exp acc}, since the baseline implementations are provided in \verb|Python| with Tensorflow. We use the \verb|C++| version to measure scalability of our parallel training in Sections \ref{sec: exp scale}, \ref{sec: exp cache} and \ref{sec: exp deep}. The \verb|C++| implementation is necessary, since \verb|Python| and Tensorflow are not flexible enough for parallel computing experiments (e.g., AVX and thread binding are not explicit in \verb|Python|). Our \verb|C++| implementation achieves comparable accuracy as the Tensorflow one. 

We run experiments on a dual 20-Core Intel\textsuperscript{\small\textregistered} Xeon E5-2698 v4 @2.2GHz machine with 512GB of DDR4 RAM. For the \texttt{Python} implementation, we use \texttt{Python} 3.6.5 with Tensorflow 1.10.0. For the \texttt{C++} implementation, the compilation is via Intel\textsuperscript{\small\textregistered} ICC (\verb|-O3| flag). ICC (version 19.0.5.281), MKL (version 2019 Update 5) and OMP are included in Intel Parallel Studio Xe 2018 update 3. 

% Datasets: ppi ($10K$); reddit ($200K$); yelp ($800K$); amazon ($1.5M$).

% Platform: 40-core Xeon

% GCN models: GCN; GCN-batched; GraphSAGE; FastGCN; Subgraph-GCN

\subsection{Evaluation on Accuracy and Efficiency}
\label{sec: exp acc}

Our graph sampling-based training significantly reduces computation complexity without accuracy loss. To eliminate the impact of different parallelization strategies on training time, here we run our implementation as well as all the baselines using single thread.  
Figure \ref{fig: exp cap 1}
plots the relation between accuracy (F1 micro score) and sequential training time. To be consistent with the settings in the original papers of the baselines, all measurements here are based on the GNN models of two GCN / GraphSAGE layers. Accuracy is measured on the validation set at the end of each epoch. Between the two baselines \citep{gcn,graphsage}, GraphSAGE \citep{graphsage} achieves higher accuracy and faster convergence. Compared with \citep{graphsage}, our minibatch training achieves higher accuracy on all graphs, {showing that our graph sampler can preserve important characteristics from the original training graph. }
{Frontier, random walk and edge sampling algorithms perform similarly on Reddit, Yelp and Amazon. On PPI, random walk and edge sampling algorithms result in lower accuracy than the frontier sampler. This is potentially due to the fact that frontier sampler preserves some graph measures better than simpler samplers such as Edge and RW \citep{frontier}. }
Due to the stochasticity in training, we define an accuracy threshold to measure training time speedup. Let $a_0$ be the highest accuracy achieved by the baselines on a given dataset. We define the accuracy threshold as $a_0-0.0025$.
%Training time speedup can vary depending on the accuracy requirement. Some applications require highest accuracy, while others may tolerate small accuracy loss with the benefit of less training time. For a dataset, let $a_0$ be the highest accuracy achieved among all the models. We define three thresholds $a_\text{high}$\footnote{Due to stochasticity in training, we allow $0.5\%$ variance in accuracy. }$=a_0-0.005$, $a_\text{mid}=a_0-0.015$ and $a_\text{low}=a_0-0.025$. 
Serial training time speedup is calculated as: the time for the best performing baseline to reach the threshold divided by the time for our model to reach the threshold. 
%To measure training time speedup, we set 3 accuracy thresholds ($a_1, a_2, a_3$, horizontal dashed lines) for each of the datasets. Let $t_i$ be the time taken by a GCN model to reach the accuracy $a_i$. Training time speedup is calculated by $t_i^\text{baseline}/t_i^\text{proposed}$.
%Table \ref{tb: exp eff} summarizes the speedup. 
%Compared with \citep{graphsage}, we achieve $1.9\times$ (PPI), $7.8\times$ (Reddit), $4.7\times$ (Yelp) and $2.1\times$ (Amazon) training time reduction. Compared with \citep{gcn}, we achieve $4.8\times$ speedup for Yelp. We cannot measure speedup for the other three datasets since \citep{gcn} does not reach the threshold. 
We achieve serial training time speedup of $1.9\times$, $7.8\times$, $4.7\times$ and $2.1\times$ for PPI, Reddit, Yelp and Amazon, respectively. 
{As stated in Section \ref{sec: exp setup}, in this set of experiments, all the runs are executed under the same Tensorflow framework using single thread. Therefore, the speedup achieved by us is not related to our parallelization strategies and is purely due to our graph sampling based training algorithm. }
Such significant speedup verifies that our minibatch training improves the computation efficiency by avoiding ``neighbor explosion'' (see Section \ref{sec: gcn efficiency}).

% 4 plots:
% \begin{itemize}
%     \item Datasets: PPI, Reddit, Yelp, Amazon
%     \item Title: N/A
%     \item $x$-axis: Single Thread Training Time (sec)
%     \item $y$-axis: Test Set F1-Micro Score
%     \item Style: Lines
%     \item Legends: GraphSAGE; FastGCN; Proposed
% \end{itemize}

%We will show that on 4 datasets, subgraph-GCN achieves highest accuracy. 

%\subsection{Comparison on Training Time}

%We will show that on 4 datasets, subgraph-GCN achieves $10\times$ training time reduction.

\subsection{Evaluation on Scalability}
\label{sec: exp scale}

In the following, we evaluate scalability of the various operations (graph sampling, feature propagation and weight transformation) in training. 

\subsubsection{Scalability of Overall Training}

\begin{figure*}[htp]
\input{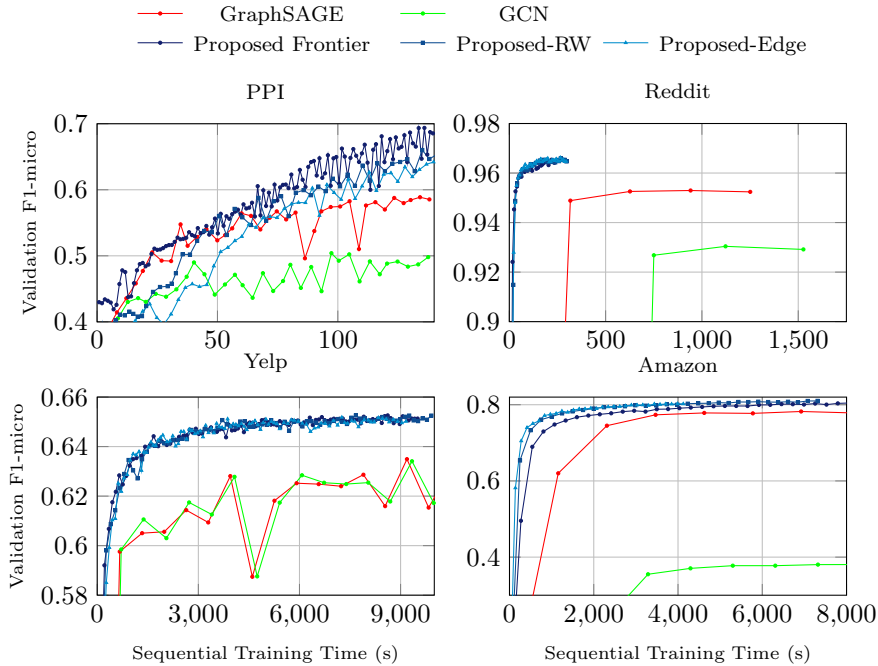}
\caption{Time-Accuracy plot for sequential execution}
\label{fig: exp cap 1}
\vspace{-0.2cm}
\end{figure*}

For the proposed GNN training, Figure \ref{fig: exp scale} shows the parallel training speedup relative to sequential execution. The execution time includes every training steps specified by lines 2 to 13 of Algorithm \ref{algo: gsaint scheduler} --- 
\begin{enumerate*}
\item {frontier} graph sampling (with AVX enabled) and subgraph transpose, 
\item feature aggregation in the forward propagation and its corresponding operation in the backward propagation,
\item weight transformation in the forward propagation and its corresponding operation in the backward propagation, and
\item all the other operations (e.g., $\func[ReLU]{}$ activation, sigmoid function, etc. ) in the forward and backward propagation
\end{enumerate*}. 
As before, we evaluate scaling on a 2-layer GraphSAGE model, with small and large hidden dimensions, $f^{(1)}=f^{(2)}=512$ and $1024$, respectively. 
As shown by the plots A and D of Figure \ref{fig: exp scale}, the overall training is highly scalable, consistently achieving around $15\times$ speedup on 40-cores for all datasets. The performance breakdown in plots G and H of Figure~\ref{fig: exp scale} suggests that
sampling time corresponds to only a small portion of the total training time. This is due to
\begin{enumerate*}
    \item low serial complexity of our Dashboard based implementation, and
    \item highly scalable implementation using intra- and inter-sampler parallelism.
\end{enumerate*}
In addition, feature aggregation for Amazon corresponds to a significantly higher portion of the total time compared with other datasets. This is due to the higher degree of the subgraphs sampled from Amazon. 
The main bottleneck in scaling is the weight transformation step performing dense matrix multiplication (see analysis in Section \ref{sec: exp weight transform}). %As discussed in Section \ref{sec: para feat prop kernel}, this step is parallelized by directly invoking the $\func[]{cblas_dgemm}$ routine of Intel MKL. 
The overall  performance scaling is also data dependent. 
For denser graphs such as Amazon, the scaling of the feature aggregation step dominates the overall scalability. For the other sparser graphs, the weight transformation step has a higher impact on the training. 
Lastly, our parallel algorithm can scale well under a wide range of configurations --- whether the hidden dimension is small or large; whether the training graph is small or large, sparse or dense.  
%High-centrality nodes in the training graph are more likely to be sampled, making caching across training iterations a non-negligible factor. The time on dense matrix multiplication becomes more significant with more processors. 

\begin{figure*}[htp]
\centering
% 	\begin{subfigure}[b]{1.\textwidth}		
%     \includegraphics[width=1\textwidth]{./Figures/scale_512.eps}
% 	\label{fig: exp scale 512}
% 	\end{subfigure}
% 	\begin{subfigure}[b]{1.0\textwidth}
% 	\includegraphics[width=1\linewidth]{./Figures/scale_1024.eps}
% 	\label{fig: exp scale 1024}
% 	\end{subfigure}
\input{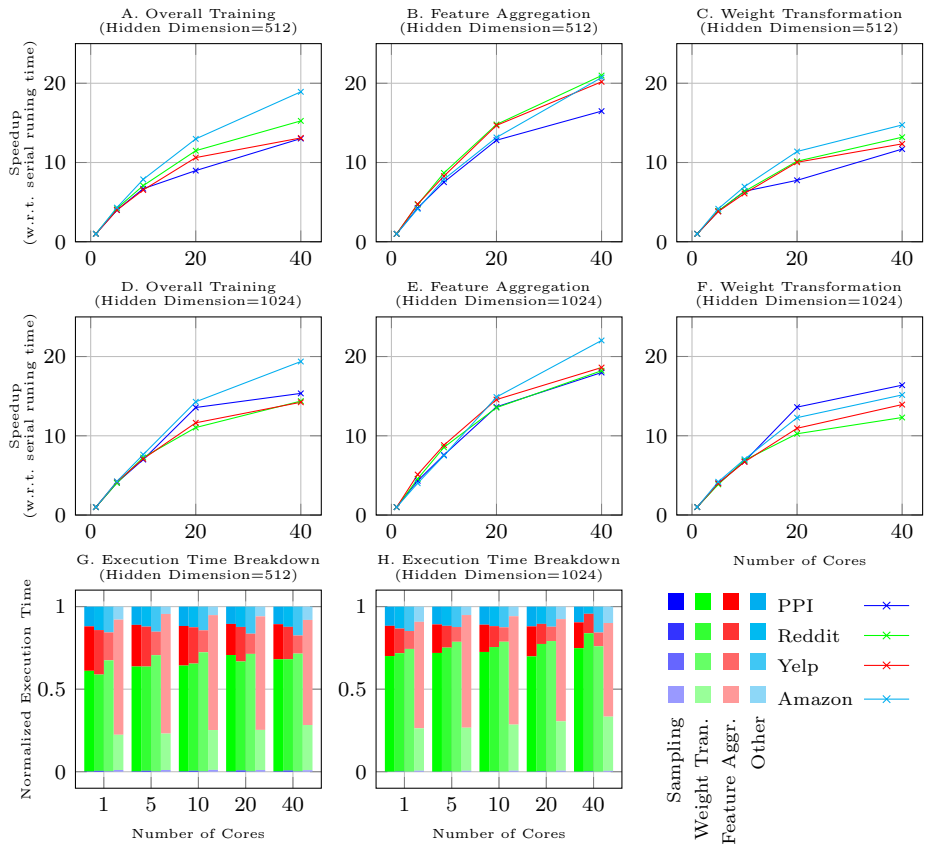}
\caption{Scaling evaluation with hidden feature dimensions: $512$  and $1024$}
\label{fig: exp scale}
\end{figure*}

\subsubsection{Scalability of Parallel Graph Sampling}

{We evaluate the effect of inter-sampler parallelism for the frontier, random walk and edge sampling algorithms, and intra-sampler parallelism for the frontier sampling algorithm.}

{For the frontier sampling algorithm, the} AVX2 instructions supported by our target platform translate to maximum of 8 intra-subgraph parallelism ($p_\text{intra}=8$). The total of 40 Xeon cores makes $1\leq p_\text{inter}\leq 40$. Figure \ref{fg:noavx_vs_avx}.A shows the effect of $p_\text{inter}$, when $p_\text{intra}=8$ (i.e., we launch $1\leq p_\text{inter}\leq 40$ independent samplers, where AVX is enabled within each sampler). Sampling is highly scalable with inter-subgraph parallelism. We observe that scaling performance degrades when going from 20 to 40 cores, {due to mixed effect of lower boost frequency and limited memory bandwidth. With all the 20 cores in one chip executing AVX2 instructions, the Xeon CPU can only boost to 2.2GHz, in contrast with 3.4GHz for executing AVX instructions only on one core. }
Figure \ref{fg:noavx_vs_avx}.B shows the effect of $p_\text{intra}$ under various $p_\text{inter}$. The bars show the speedup of using AVX instructions comparing with otherwise. We achieve around $4\times$ speedup on average. The scaling on $p_\text{intra}$ is data dependent. Depending on the training graph degree distribution, there may be significant portion of nodes with less than 8 neighbors, resulting in under-utilization of the AVX2 instruction. {We can understand such under-utilization of instruction-level parallelism as a result of load-imbalance due to node degree variation. Such load-imbalance explains the discrepancy from the theoretical modeling on the sampling scalability (Theorem \ref{thm: frontier scale})}.  

{Figure \ref{fg:noavx_vs_avx}.C and \ref{fg:noavx_vs_avx}.D show the effect of $p_\text{inter}$ for random walk and edge sampling algorithms. Both sampling algorithms scale more than 20$\times$ when $p_\text{inter}=40$. As we do not use AVX instructions for thse two samplers (i.e., $p_\text{intra}=1$, and the CPU frequency is unaffected), the scalability from 20 cores to 40 cores is better than that of the frontier sampler. }
%For all datasets and $p_\text{inter}$, we achieve around 8 times speedup by intra-subgraph parallelism, consistent with the analysis in Section \ref{sec: para sample inter}.
%Figures \ref{fg:noavx_vs_avx}A, B show the effect of intra- and inter-parallelism respectively, using our Dashboard-based implementation. 
%Figure \ref{fg:noavx_vs_avx}A shows the speedup of per-subgraph sampling time with respect to $p_\text{inter}$ (we enable AVX2 in this case, and thus $p_\text{intra}=8$). Our parallel frontier sampling algorithms scales up to 20.69x (Amazon dataset) with 40 inter threads each assigned to one physical core in the machine. Figure \ref{fg:noavx_vs_avx}B shows the speedup under different number of inter-threads with intra-thread=1(no intra-thread) and intra-thread=8(AVX2). On average, $7.51\times$ speedup is achieved through the benefit of intra-thread parallelism.
%%%%As a side-note, for highly skewed graphs such as Amazon, we allocate no more than 30 DB entries to a vertex (regardless of its actual degree), in order to upper bound the probability of popping it from the frontier. This helps accuracy improvement on graphs with skewed degree distributions, since it prevents the situation where all subgraphs contain mostly the same set of vertices.

\begin{figure}
\centering
\begin{tikzpicture}

    \pgfplotsset{compat=newest,scaled x ticks=false}

    \def\widthplot{0.5\textwidth}
    \def\titlefontsize{\footnotesize}
    \def\axisfontsize{\scriptsize}
    \def\labelfontsize{\footnotesize}

    \begin{groupplot}[group style={group size= 2 by 2,horizontal sep=0.15\textwidth, vertical sep=0.15\textwidth},height=0.4\textwidth]
        \nextgroupplot[ylabel={\axisfontsize Speedup},xlabel={\axisfontsize $p_\text{inter}$},title={\titlefontsize A. Frontier Sampling Speedup ($p_\text{intra}=8)$},width=\widthplot,grid,ymax=30]
            \addplot[mark=x,color=blue]
                coordinates{( 1 , 1.0 )( 5 , 4.925 )( 10 , 6.840277777777778 )( 20 , 12.429022082018927 )( 40 , 15.6039603960396 )};
            \label{ppi}
            \addplot[mark=x,color=green]
                coordinates{( 1 , 1.0 )( 5 , 4.696969696969696 )( 10 , 7.848101265822785 )( 20 , 12.19672131147541 )( 40 , 15.152749490835031 )};
            \label{reddit}
            \addplot[mark=x,color=red]
                coordinates{( 1 , 1.0 )( 5 , 4.481865284974093 )( 10 , 8.277511961722489 )( 20 , 11.849315068493151 )( 40 , 14.151329243353782 )};
            \label{yelp}
            \addplot[mark=x,color=cyan]
                coordinates{( 1 , 1.0 )( 5 , 4.642857142857143 )( 10 , 8.235294117647058 )( 20 , 11.704180064308682 )( 40 , 16.14190687361419 )};
            \label{amazon}
            \coordinate (top) at (rel axis cs:0,1);
        \nextgroupplot[ybar=0cm,symbolic x coords={1,5,10,20,40},enlarge x limits=true,xtick=data,ylabel={\axisfontsize Speedup},xlabel={\axisfontsize $p_\text{inter}$},title={\titlefontsize B. Frontier Sampling AVX Speedup},width=\widthplot,xtick align=inside,ymajorgrids,ymin=0,ymax=6,bar width=0.18cm]
            \addplot[draw=none,fill=blue]
                coordinates{(1, 4.441624365482234 )(5, 4.63 )(10, 5.458333333333334 )(20, 5.482649842271294 )(40, 3.948514851485148 )};
            \label{ppiavx}
            \addplot[draw=none,fill=green]
                coordinates{(1,4.129032258064516)(5,4.242424242424242)(10,4.1645569620253164)(20,3.488524590163934)(40,3.782077393075357)};
            \label{redditavx}
            \addplot[draw=none,fill=red]
                coordinates{(1,3.913294797687861)(5,3.8186528497409324)(10,4.1913875598086126)(20,4.6506849315068495)(40,3.353783231083845)};
            \label{yelpavx}
            \addplot[draw=none,fill=cyan]
                coordinates{(1,3.4505494505494503)(5,3.5408163265306127)(10,3.701357466063348)(20,4.032154340836013)(40,3.35920177383592)};]
            \label{amazonavx}
            \coordinate (bot) at (rel axis cs:1,0);
        \nextgroupplot[ylabel={\axisfontsize Speedup},xlabel={\axisfontsize $p_\text{inter}$},title={\titlefontsize A. Random Walk Sampling Speedup},width=\widthplot,grid,ymax=30]
            \addplot[mark=x,color=blue]
                coordinates{( 1 , 1.0 )( 5 , 4.32 )( 10 , 7.43 )( 20 , 14.54 )( 40 , 21.55 )};
            \label{ppi}
            \addplot[mark=x,color=green]
                coordinates{( 1 , 1.0 )( 5 , 4.53 )( 10 , 7.75 )( 20 , 15.68 )( 40 , 24.05 )};
            \label{reddit}
            \addplot[mark=x,color=red]
                coordinates{( 1 , 1.0 )( 5 , 4.29 )( 10 , 7.41 )( 20 , 14.46 )( 40 , 21.92 )};
            \label{yelp}
            \addplot[mark=x,color=cyan]
                coordinates{( 1 , 1.0 )( 5 , 4.43 )( 10 , 7.02 )( 20 , 13.91 )( 40 , 24.40 )};
            \label{amazon}
        \nextgroupplot[ylabel={\axisfontsize Speedup},xlabel={\axisfontsize $p_\text{inter}$},title={\titlefontsize A. Edge Sampling Speedup},width=\widthplot,grid,ymax=30]
            \addplot[mark=x,color=blue]
                coordinates{( 1 , 1.0 )( 5 , 4.40 )( 10 , 7.45 )( 20 , 14.84 )( 40 , 22.05 )};
            \label{ppi}
            \addplot[mark=x,color=green]
                coordinates{( 1 , 1.0 )( 5 , 4.34 )( 10 , 7.51 )( 20 , 15.11 )( 40 , 22.57 )};
            \label{reddit}
            \addplot[mark=x,color=red]
                coordinates{( 1 , 1.0 )( 5 , 4.00 )( 10 , 7.14 )( 20 , 14.15 )( 40 , 20.17 )};
            \label{yelp}
            \addplot[mark=x,color=cyan]
                coordinates{( 1 , 1.0 )( 5 , 4.40 )( 10 , 7.55 )( 20 , 14.98 )( 40 , 23.74 )};
            \label{amazon}
    \end{groupplot}
    \path (top|-current bounding box.south)--
        coordinate(legendpos)
        (bot|-current bounding box.south);
    \matrix[matrix of nodes,anchor=south,inner sep=0.2em,draw=none] at ([yshift=-3.5ex]legendpos)
    {
        \ref{ppi} \ref{ppiavx} & \labelfontsize PPI & [5pt]
        \ref{reddit} \ref{redditavx} & \labelfontsize Reddit & [5pt]
        \ref{yelp} \ref{yelpavx} & \labelfontsize Yelp & [5pt]
        \ref{amazon} \ref{amazonavx} & \labelfontsize Amazon \\
    };
\end{tikzpicture}
\caption{Sampling speedup (inter- \& intra-subgraph parallelism)}
\label{fg:noavx_vs_avx}
\end{figure}
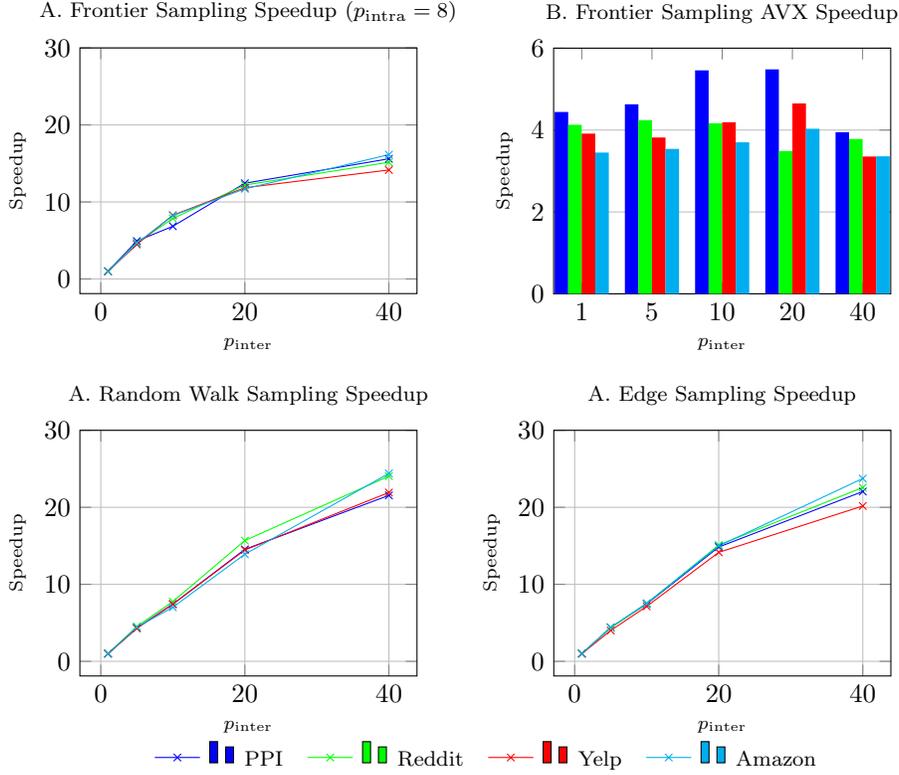

% 2 plots:
% \begin{itemize}
%     \item Datasets: PPI, Reddit, Yelp, Amazon
%     \item Title: N/A
%     \item $x$-axis: Number of Threads
%     \item $y$-axis: Sampling Time Per Subgraph (sec)
%     \item Style: Lines
%     \item Legends: $\eta=2,4,8, p_{\text{intra}}=1,2,4,8$
% \end{itemize}

\subsubsection{Scalability of Feature Aggregation}

Figure \ref{fig: exp scale} shows the scalability of the feature aggregation step using our partitioning strategy. We achieve good scalability (around $20\times$ speedup on 40 cores) for all datasets under various feature sizes, thanks to our caching strategy and the optimal load-balance discussed in Section \ref{sec: para feat prop}. 
According to the analysis, the scalability of feature aggregation should not be significantly affected by the subgraph topological characteristics. Therefore, we observe from plots B and E of Figure \ref{fig: exp scale} that, the curves for the four datasets look similar to each other. 
%Note that the feature aggregation operation even achieves superlinear speedup on Reddit and Yelp. This is due to the caching effect across training iterations. 

%In addition, the average measured performance when using 40 cores is $d\text{GFLOPS}$, which is close to the system peak performance $e\text{GFLOPS}$. This verifies our performance bound analysis in Section \ref{sec: para feat prop}. 

%2 plots:
%\begin{itemize}
%    \item Datasets: PPI, Reddit, Yelp, Amazon
%    \item Title: N/A
%    \item $x$-axis: Number of Threads
%    \item $y$-axis: Feature Propagation Time (sec)
%    \item Style: Lines; bars
%    \item Legends: With/Without intelligent partitioning; $p_{\text{feat}}=1,2,4,8$
%\end{itemize}

\subsubsection{Scaling of Weight Transformation}
\label{sec: exp weight transform}
As discussed in Section \ref{sec: para feat prop kernel}, the weight transformation operation is implemented by \texttt{cblas\_dgemm} routine of the Intel\textsuperscript{\small\textregistered} MKL \citep{MKL} library. All optimizations on the dense matrix multiplication are internally implemented in the library. Plots C and F of Figure \ref{fig: exp scale} show the scalability result. 
On 40 cores, average of $13\times$ speedup is achieved. We speculate that the overhead of MKL's internal thread and buffer management is the bottleneck on further scaling.

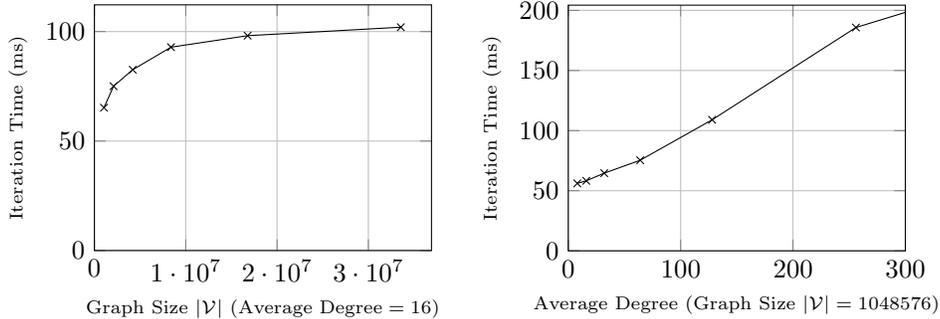
\begin{figure}
	\centering
	\begin{tikzpicture}

    \pgfplotsset{compat=newest,scaled x ticks=false}

    \def\widthplot{0.5\textwidth}
    \def\titlefontsize{\footnotesize}
    \def\axisfontsize{\scriptsize}
    \def\labelfontsize{\footnotesize}

    \begin{groupplot}[group style={group size= 2 by 1,horizontal sep=0.15\textwidth},grid,height=0.4\textwidth]
        \nextgroupplot[ylabel={\axisfontsize Iteration Time (ms)},xlabel={\axisfontsize Graph Size $\size{\V}$ ($\text{Average Degree}=16$)},width=\widthplot,xmin=0,ymin=0]
            \addplot[mark=x]
                coordinates{(1048576,65.25)(2097152,75.05)(4194304,82.575)(8388608,92.875)(16777216,98.125)(33554432,102)};
        \nextgroupplot[ylabel={\axisfontsize Iteration Time (ms)},xlabel={\axisfontsize Average Degree (Graph Size $\size{\V}=1048576$)},width=\widthplot,xmin=0,xmax=300,ymin=0]
            \addplot[mark=x]
                coordinates{(8,56.1)(16,58.25)(32,64.6)(64,75.35)(128,108.975)(256,185.75)(512,258.85)(1024,402.8)};
    \end{groupplot}

\end{tikzpicture}
	\caption{Training Time in Synthetic Graph}
	\label{fig: exp synthetic}
\end{figure}

% \subsubsection{Performance breakdown}

% Figure \ref{fig: } shows the performance breakdown under various number of cores. Firstly, sampling time corresponds to only a small portion of the total time. This is due to
% \begin{enumerate*}
%     \item low serial complexity of our Dashboard based implementation;
%     \item highly scalable design using intra- and inter-subgraph parallelism.
% \end{enumerate*}
% Secondly, time on dense matrix multiplication becomes more significant with more processors. This indicates that the scaling bottleneck is in the \texttt{dgemm} implementation internal to the library.  

\subsection{Effect of Cache Size}
\label{sec: exp cache}

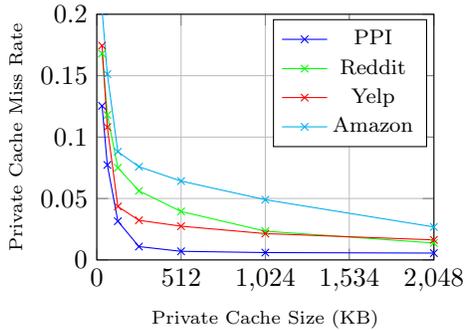
\begin{figure}
	\centering
	\begin{tikzpicture}

    \pgfplotsset{compat=newest,scaled x ticks=false,scaled y ticks=false}

    \def\widthplot{0.5\textwidth}
    \def\titlefontsize{\footnotesize}
    \def\axisfontsize{\scriptsize}
    \def\labelfontsize{\footnotesize}

    \begin{axis}[xlabel={\axisfontsize Private Cache Size (KB)},ylabel={\axisfontsize  Private Cache Miss Rate},grid=major,height=0.8*\widthplot,width=\widthplot,legend style={font=\labelfontsize},xmin=0,xmax=2048,ymin=0,ymax=0.2,xtick={0,512,1024,1534,2048},ytick={0,0.05,0.1,0.15,0.2},yticklabel style={/pgf/number format/fixed}]
        \addplot[color=blue,mark=x]
            coordinates{( 32 , 0.12542799103367672 )( 64 , 0.07731177377626071 )( 128 , 0.031498005752695145 )( 256 , 0.010858771161916076 )( 512 , 0.007016346014113375 )( 1024 , 0.0059781800385870585 )( 2048 , 0.005571232739299103 )};
            \addlegendentry{\labelfontsize PPI}
        \addplot[color=green,mark=x]
            coordinates{( 32 , 0.16799870045017967 )( 64 , 0.11796867399757792 )( 128 , 0.07504750401877427 )( 256 , 0.05622569086193041 )( 512 , 0.039444809448541125 )( 1024 , 0.023401644578458326 )( 2048 , 0.013791471328992172 )};
            \addlegendentry{\labelfontsize Reddit}
        \addplot[color=red,mark=x]
            coordinates{( 32 , 0.17444218411117088 )( 64 , 0.10809646347109458 )( 128 , 0.043514730618989944 )( 256 , 0.032257962707229665 )( 512 , 0.027467001463634868 )( 1024 , 0.02142076546560424 )( 2048 , 0.016314342577477492 )};
            \addlegendentry{\labelfontsize Yelp}
        \addplot[color=cyan,mark=x]
            coordinates{( 32 , 0.20364752164966324 )( 64 , 0.1512277354496316 )( 128 , 0.08807188985630125 )( 256 , 0.07587845780960793 )( 512 , 0.06421060189651843 )( 1024 , 0.049012438038222465 )( 2048 , 0.026849472736495347 )};
            \addlegendentry{\labelfontsize Amazon}
    \end{axis}

\end{tikzpicture}
	\caption{L2 Cache Miss Rate}
    \label{fig: exp cache}
\end{figure}

Since our partitioning strategy for feature aggregation {(Section \ref{sec: para feat prop})} is based on the L2-cache size of the system, we evaluate the cache miss rate under various cache sizes by simulation. {We use CSR format to represent the sparse adjacency matrix of the subgraph and column major layout to represent the dense feature matrix $\X_s$.} We use the open-source simulator DynamoRIO \citep{dynamorio} to simulate our \verb|C++| implementation. 
We configure the system to be 40 cores with two levels of cache, {where the first level of cache corresponds to the L2-cache of the real system}. We vary the size of the first level of private cache from 32KB to 2048KB. We fix the size of the second level of shared cache to be 50MB.
We let the simulator to run one full training iteration and record the cache miss rate for the first level of private cache.
% We let the simulator to run the full training iteration and calculate the cache miss rate as follows: we first record the number of cache hits $h_1$ and cache misses $m_1$ for the full training iteration, and then record the number of cache hits $h_2$ and cache misses $m_2$ for the 
Figure \ref{fig: exp cache} shows the effect of cache size on cache miss rate. When the cache size increases from 32KB to 512KB, the cache miss rate quickly drops to below 5\%. 
%Note that the L2-cache size of our target platform is 256KB. Therefore, 
The parallel execution using our partitioning strategy indeed leads to low cache miss rate. {This indicates small amount of slow-to-fast memory data traffic as a benefit of our partitioning strategy. }

\subsection{Comparison with GPU}

{We compare the proposed training algorithm with GPU implementation from Tensorflow. We run the GPU program on an Nvidia Tesla P100 GPU with 16GB of GDDR5 memory, with the same Xeon CPU server as described in Section \ref{sec: exp setup}. Table \ref{tab:GPU} shows the performance of the proposed training algorithm on CPU and the Tensorflow implementation on GPU. Both use the same parallel graph sampling algorithm as described in Section \ref{sec: para sample}. For the CPU execution, we use all the available 40 cores. For GPU program, the sampling is done on CPU with 40 cores, while the rest parts are done on GPU. We use the frontier sampling algorithm with node budget $n=8000$ and $p_\text{intra}=8$. We choose hidden dimension $f=512$ and record the average execution time per iteration for 100 iterations. The GPU program runs faster than the CPU program by 1.93$\times$, 2.71$\times$, 2.05$\times$ and 2.20$\times$ on PPI, Reddit, Yelp and Amazon dataset. Note that the peak performance of the CPUs is only 3.5 TFLOPS while the peak performance of the GPU is 10.3 TFLOPS. As stated in Section \ref{sec: para spmm}, the proposed parallel training algorithm scales up to 136 cores on CPU. On a 64- or 128-core machine, the proposed algorithm would out-perform GPU based on our modeling (Section \ref{sec: para feat prop}). 
Importantly, the fast training on GPU also indicates the effectiveness of our graph sampling based minibatch algorithm as well as our parallelization strategy on the frontier sampler. 
}

\begin{table}[!ht]
\caption{Execution Time (s) Per Iteration (Hidden Dimension = 512)}
    \centering
    \begin{tabular}{rcc}
        \toprule
        Dataset & CPU & GPU\\
        \midrule
        \midrule
        PPI & 0.1974 & 0.1021\\
        Reddit & 0.3676 & 0.1357\\
        Yelp & 0.2917 &  0.1420\\
        Amazon & 0.4416 & 0.2004\\
        \bottomrule
    \end{tabular}
    \label{tab:GPU}
\end{table}

\subsection{Evaluation on Synthetic Graphs}

Since the largest available real-world dataset for GNN training (i.e., Amazon) contains only about 1.5 million nodes, we generate synthetic graphs of much larger sizes to perform more thorough scalability evaluation. 
In the left plot of Figure \ref{fig: exp synthetic}, the sizes of the synthetic graphs grow from 1 million nodes to around 33 million nodes. All synthetic graphs have average degree of 16. We run a 2-layer GNN with hidden dimension of 512 on the subgraphs of the synthetic graphs. The vertical axis denotes the time to compute one iteration (i.e., the time to perform forward and backward propagation on one minibatch subgraph). 
{The subgraphs are all sampled by the frontier sampling algorithm with the same sampling parameters of $n=8000$ and $m=1000$.}
With the increase of the training graph size, the iteration time converges to a constant value of around 100 ms. This indicates that our parallel training is highly scalable with respect to the graph size. {When increasing the number of graph nodes, we keep the average degree unchanged. Therefore, the degree of the sampled subgraphs also keeps unchanged (due to the property of frontier sampling). Since we set the node budget $n$ to be fixed, the subgraph size (in terms of number of nodes and edges) in each iteration is approximately independent of the total number of nodes in the training graph. 
So the cost to perform one step of gradient update does not depend on the training graph size (for a given training graph degree).
}

In the right plot of Figure \ref{fig: exp synthetic}, we fix the graph size as $\size{\V}=2^{20}$ and increase the average degree. 
Under the same sampling algorithm, if the original graph becomes denser, the sampled subgraphs are more likely to be denser as well. 
The computation complexity of feature aggregation is proportional to the subgraph degree. We observe that the iteration time approximately grows linearly with the average degree of the original training graph. This indicates that our parallel training algorithm can handle both sparse and dense graphs very well.

\subsection{Deeper Learning}
\label{sec: exp deep}

Although state-of-the-art training methods \citep{graphsage,fastgcn,asgcn,sgcn} are not evaluated on GNN models deeper than 3 layers, adding more layers in a neural network is proven to be very effective in increasing the expressive power (and thus accuracy) of the network \citep{nn_depth}. 
Here we evaluate the efficiency and overall training speedup of our GNN implementation compared with \citep{graphsage}, under various number of layers using 40 processors. 
The evaluation is based on our \verb|C++| implementation. 

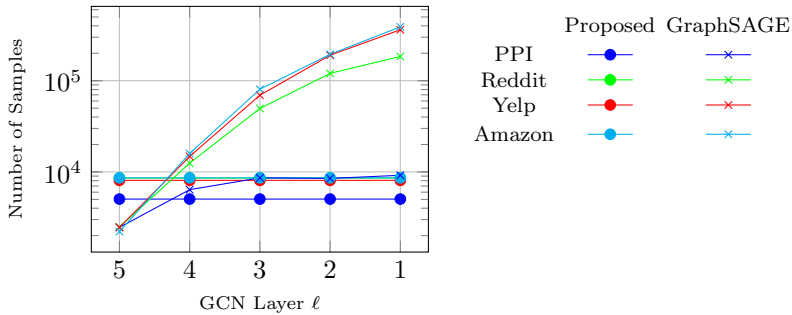
\begin{figure}
	\centering
	\begin{tikzpicture}

    \pgfplotsset{compat=newest,scaled x ticks=false,scaled y ticks=false}

    \def\widthplot{0.5\textwidth}
    \def\titlefontsize{\footnotesize}
    \def\axisfontsize{\scriptsize}
    \def\labelfontsize{\footnotesize}

    \begin{axis}[ylabel={\axisfontsize Number of Samples},xlabel={\axisfontsize GCN Layer $\ell$},grid=major,height=0.8*\widthplot,width=\widthplot,legend style={font=\labelfontsize},yticklabel style={/pgf/number format/fixed},ymode=log,xtick={1,2,3,4,5},xticklabels={5,4,3,2,1}]
        \addplot[color=blue,mark=*]
            coordinates{(1,5037)(2,5037)(3,5037)(4,5037)(5,5037)};
            \label{gsaintppi}
        \addplot[color=green,mark=*]
            coordinates{(1,8565)(2,8565)(3,8565)(4,8565)(5,8565)};
            \label{gsaintreddit}
        \addplot[color=red,mark=*]
            coordinates{(1,8082)(2,8082)(3,8082)(4,8082)(5,8082)};
            \label{gsaintyelp}
        \addplot[color=cyan,mark=*]
            coordinates{(1,8606)(2,8606)(3,8606)(4,8606)(5,8606)};
            \label{gsaintamazon}
        \addplot[color=blue,mark=x]
            coordinates{(1,2466)(2,6385)(3,8604)(4,8490)(5,9180)};
            \label{gsageppi}
        \addplot[color=green,mark=x]
            coordinates{(1,2485)(2,12475)(3,49632)(4,120315)(5,184523)};
            \label{gsagereddit}
        \addplot[color=red,mark=x]
            coordinates{(1,2465)(2,14787)(3,69322)(4,190336)(5,362250)};
            \label{gsageyelp}
        \addplot[color=cyan,mark=x]
            coordinates{(1,2217)(2,15923)(3,80661)(4,194837)(5,389734)};
            \label{gsageamazon}
        \coordinate (bot) at (rel axis cs:1,1);
    \end{axis}

    \matrix[matrix of nodes,anchor=north west,inner sep=0.2em,draw=none,font=\labelfontsize,column 5/.style={anchor=base west}] at ([xshift=3ex,yshift=0]bot)
    {
        & Proposed & GraphSAGE \\
       PPI & \ref{gsaintppi} & \ref{gsageppi} \\
       Reddit & \ref{gsaintreddit} & \ref{gsagereddit} \\
       Yelp & \ref{gsaintyelp} & \ref{gsageyelp} \\
       Amazon & \ref{gsaintamazon} & \ref{gsageamazon} \\
    };

\end{tikzpicture}
	\caption{Comparison on the number of sampled nodes per GNN layer}
    \label{fig: nodeslayer}
\end{figure}

We first evaluate the computation efficiency. As discussed in Section \ref{sec: gcn efficiency}, layer sampling based training methods such as \citep{graphsage} suffer from ``neighbor explosion''. Therefore, on deep models, there may be significant amount of redundant computation across training iterations. Recall that we analyze the per epoch computation complexity in Section \ref{sec: gcn efficiency}, under the two cases of large and small batch sizes respectively. 
Figure \ref{fig: nodeslayer} shows the severity of ``neighbor explosion'' by visualizing the number of sampled nodes per GNN layer for the two training methods. Denote $L$ as number of graph convolution layers. 
The minibatch sampling of \citep{graphsage} proceeds as follows. \citep{graphsage} first randomly pick the $r$ number of root nodes from the output of the last graph convolution layer (i.e., layer-$L$). Then, to generate the layer $\ell-1$ samples, it randomly pick $s^{\paren{\ell}}$ neighbors of each layer $\ell$ sampled nodes. \citep{graphsage} completes the minibatch construction when it has finished picking the input nodes of layer $1$. 
Following the recommended setting of \citep{graphsage}, we set $r=512$, $s^{\paren{L}}=25$ and $s^{\paren{\ell}}=10$ for $1\leq \ell\leq L-1$. 
Regarding our proposed training algorithm, since the sampling is performed on the training graph rather than the GNN, all layers have the same $\size{\V[s]}$ number of nodes. 
Figure \ref{fig: nodeslayer} shows the number of unique sampled nodes per layer for the two training methods. 
When the GNN model is deep, \citep{graphsage} requires orders of magnitude more samples than our training method. 
In addition, the number of sampled nodes of \citep{graphsage} eventually converges to the full graph size $\size{\V}$ when the GNN depth is high. In summary, Figure \ref{fig: nodeslayer} empirically verifies the complexity analysis in Section \ref{sec: gcn efficiency} and shows the advantage in high training efficiency of our method.

\begin{figure}
	\centering
	\begin{tikzpicture}

    \pgfplotsset{compat=newest,scaled x ticks=false,scaled y ticks=false}

    \def\widthplot{0.5\textwidth}
    \def\titlefontsize{\footnotesize}
    \def\axisfontsize{\scriptsize}
    \def\labelfontsize{\footnotesize}

    \begin{axis}[ylabel={\axisfontsize Normalized Training Time},xlabel={\axisfontsize GCN Depth},grid=major,height=0.8*\widthplot,width=\widthplot,legend style={font=\labelfontsize},yticklabel style={/pgf/number format/fixed},xtick={1,2,3,4},ymax=100],%ymode=log,
        \addplot[color=blue,mark=*]
            coordinates{(1,1)(2,3.77)(3,6.37)(4,8.84)};
            \label{gsaintppi}
        \addplot[color=green,mark=*]
            coordinates{(1,1)(2,2.12)(3,3.24)(4,4.27)};
            \label{gsaintreddit}
        \addplot[color=red,mark=*]
            coordinates{(1,1)(2,2.77)(3,3.47)(4,4.95)};
            \label{gsaintyelp}
        \addplot[color=cyan,mark=*]
            coordinates{(1,1)(2,3.52)(3,5.83)(4,8.27)};
            \label{gsaintamazon}
        \addplot[color=blue,mark=x]
            coordinates{(1,1)(2,3.72)(3,33.31)(4,417.65)};
            \label{gsageppi}
        \addplot[color=green,mark=x]
            coordinates{(1,1)(2,27)(3,248.31)};
            \label{gsagereddit}
        \addplot[color=red,mark=x]
            coordinates{(1,1)(2,8.71)};
            \label{gsageyelp}
        \addplot[color=cyan,mark=x]
            coordinates{(1,1)(2,6.39)};
            \label{gsageamazon}
        \coordinate (bot) at (rel axis cs:1,1);
    \end{axis}

    \matrix[matrix of nodes,anchor=north west,inner sep=0.2em,draw=none,font=\labelfontsize,column 5/.style={anchor=base west}] at ([xshift=3ex,yshift=0]bot)
    {
        & Proposed & GraphSAGE \\
       PPI & \ref{gsaintppi} & \ref{gsageppi} \\
       Reddit & \ref{gsaintreddit} & \ref{gsagereddit} \\
       Yelp & \ref{gsaintyelp} & \ref{gsageyelp} \\
       Amazon & \ref{gsaintamazon} & \ref{gsageamazon} \\
    };

\end{tikzpicture}
	\caption{Comparison of training time on deep GNN models}
    \label{fig: exp deeper time}
\end{figure}
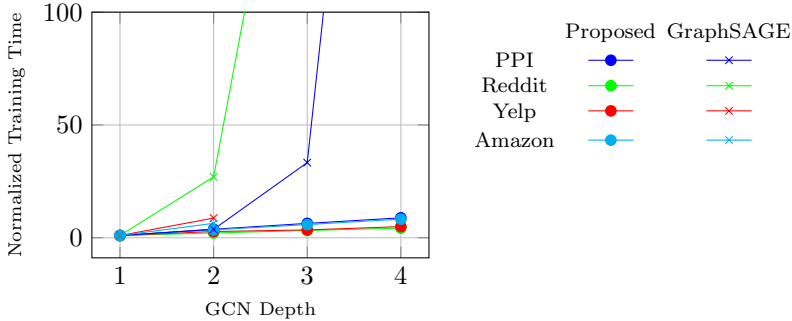

We further compare the overall training time for deep GNN models. As shown in Figure \ref{fig: exp deeper time}, we increase the GNN depth from $L=1$ to $L=4$, and set the sampling parameters as described in the above paragraph. 
Execution of both training methods uses all the 40 processing cores. 
We do not consider the difference in convergence rate and thus only measures the per-iteration execution time. 
We normalize the training time by setting the 1-layer GNN execution time as 1. 
When $L\geq 3$, 
the implementation of \citep{graphsage} results in prohibitively high training cost on PPI and Reddit, and throws runtime error on Yelp and Amazon. 
On the other hand, the training time of our method scales almost linearly with respect to the model depth. We conclude that our minibatch training algorithm, together with the parallelization and scheduling techniques, significantly facilitate the development and deployment of deeper GNN models.

\section{Discussion}

This work proposed co-design of the GNN minibatch training algorithm and the corresponding parallelization strategy. 
We next discuss several potential extensions to our parallel training algorithm. 

\paragraph{Hardware acceleration}
Our minibatch training algorithm can be used to facilitate hardware accelerator design as well. 
Apart from higher computation efficiency, another benefit of constructing minibatches by subgraphs is the reduction in communication cost. 
Suppose we use a resource-constrained hardware accelerator (e.g., FPGA) to speedup GNN training. We can sample small subgraphs so that the features of the subgraph nodes fit in the on-chip memory (whose typical size is tens of mega bits). Each iteration, once the input node features of the subgraph is transferred on-chip, the FPGA can perform the full forward and backward propagation without any communication to the external DDR memory. Therefore, we potentially achieve close-to-peak computation performance on the FPGA. 
The work in \citep{graphact} has developed a high-performance accelerator on the CPU-FPGA heterogeneous platform using our graph sampling based training algorithm. 
They quantify the feasibility of implementing the various training algorithms \citep{graphsage,fastgcn,asgcn,sgcn} on hardware by a metric called computation-communication ratio $\gamma$, where higher value of $\gamma$ indicates lower overhead in external memory communication. They further show that our algorithm achieves significantly higher $\gamma$ than the other methods \citep{fastgcn,asgcn,graphsage,sgcn}. 

\paragraph{Distributed processing}
The graph sampling based minibatch training is suitable to be executed in the distributed environment. 
After partitioning the training graph in distributed memory, each processing node can perform graph sampling independently on the local partition. Afterwards, forward and backward propagation can be executed without data access to the remote memory. 
In order to ensure convergence quality, shuffling of the node and edge data is required during the training. The optimal shuffling probability may then be derived given the graph sampling algorithm and the connectivity among the processing nodes. It is worth noticing that on each processing node, we can still locally speedup the forward and backward layer computation by designing hardware accelerators or using the parallelization strategy shown in this paper. 

% \paragraph{Other samplers}

% Our minibatch training algorithm can be extended by designing new graph sampling algorithms. Based on the two principles in Section \ref{sec: gcn accuracy} to design the $\func[SAMPLE]{\G}$ function, the minibatch subgraphs obtained from other sampling processes may lead to high training accuracy as well. 
% The work in \citep{graphsaint} proposed four light-weight topology based graph samplers that fit in our training and scheduling framework. 
% An interesting future direction is to develop specific sampling algorithms for various types of graphs (e.g., power-law graphs, planar graphs), so that the overall training accuracy can be further improved. 

% \paragraph{Other GCN architectures}
% While the forward and backward propagation rules shown in Equations \ref{eq: gcn forward} and \ref{eq: gcn backward} are specific to the layer architecture of \citep{graphsage}, our proposed parallelization strategies are applicable to a wide range of other GCN architectures. 
% For the vanilla GCN \citep{gcn}, graph attention network (GAT) \citep{gat}, jumping-knowledge net (JK-net) \citep{jk} and skip-connection based GCNs \citep{asgcn,high-order}, the key computation kernels are still the two operations of feature aggregation and weight transformation. 
% Therefore, we can parallelize the training of the above GCN variants with just minor modification to the runtime scheduling. 

\section{Conclusion and Future Work}

We presented an accurate, efficient and scalable GNN training method. Considering the redundant computation incurred in state-of-the-art GNN training, we proposed a graph sampling-based minibatch algorithm which ensures accuracy and efficiency by resolving the ``neighbor explosion'' challenge. 
We further proposed parallelization techniques and a runtime scheduler to scale the graph sampling and overall training to large number of processors. 

We will extend our graph sampling based training by integrating other graph sampling algorithms and evaluating their impact on learning accuracy. We will also work on the theoretical foundation to understand the convergence property of the graph sampling based minibatch training. 
\section*{Acknowledgement}
This work is supported by the U.S. National Science Foundation (NSF) under grants OAC-1911229 and CCF-1919289. 
\bibliographystyle{elsarticle-num}
\bibliography{citation}

\end{document}